%% file: neurips_2026.tex
\documentclass{article}

\usepackage[main, final]{neurips_2026}

\usepackage[utf8]{inputenc} % allow utf-8 input
\usepackage[T1]{fontenc}    % use 8-bit T1 fonts
\usepackage{hyperref}       % hyperlinks
\usepackage{url}            % simple URL typesetting
\usepackage{booktabs}       % professional-quality tables
\usepackage{amsfonts}       % blackboard math symbols
\usepackage{nicefrac}       % compact symbols for 1/2, etc.
\usepackage{microtype}      % microtypography
\usepackage{xcolor}         % colors

\usepackage{wrapfig}
\usepackage{graphicx}
\usepackage{subcaption}
\usepackage[table]{xcolor}
\definecolor{gold}{RGB}{255,215,0}
\definecolor{silver}{RGB}{192,192,192}
\definecolor{bronze}{RGB}{205,127,50}

\usepackage{amsmath}
\usepackage{amssymb}
\usepackage{mathtools}
\usepackage{amsthm}
\usepackage{bm}
\usepackage[capitalize,noabbrev]{cleveref}
\usepackage{thm-restate}
\usepackage{pifont}

\usepackage{multirow}

\usepackage[utf8]{inputenc}
\usepackage{algorithm}
\usepackage{algorithmic}
\usepackage{subcaption}
\usepackage{wrapfig}

\input{symbols}
\theoremstyle{plain}

\theoremstyle{definition}

\theoremstyle{remark}

\usepackage{subcaption}
\usepackage[textsize=tiny]{todonotes}

\usepackage[toc,page]{appendix}
\usepackage{tocloft}  % optional: nicer ToC formatting
\usepackage[most]{tcolorbox}

\newtcolorbox{modelbox}[1]{
  colback=gray!6,
  colframe=black!60,
  arc=5pt,
  boxrule=0.6pt,
  left=0pt,
  right=0pt,
  top=0pt,
  bottom=0pt,
  enhanced,
  title=\textbf{#1},
  fonttitle=\bfseries
}

\newtcolorbox{samplebox}{
  colback=white,
  colframe=gray!50,
  arc=5pt,
  boxrule=0.4pt,
  left=0pt,
  right=0pt,
  top=0pt,
  bottom=0pt,
  enhanced,
  fontupper=\small
}

\usepackage{listings}
\usepackage{xcolor}

\usepackage{tcolorbox}
\tcbset{
  highlightcompact/.style={
    enhanced,
    frame hidden,        % removes black border
    colback=blue!4!white, % subtle light-blue background
    boxrule=0.2pt,
    arc=2pt,
    left=6pt,
    right=6pt,
    top=2pt,
    bottom=2pt,
    width=\linewidth,
    before skip=6pt,
    after skip=6pt,
    halign=center,          % centers text horizontally
    valign=center,
  }
}
\title{LangFlow: Continuous Diffusion Rivals Discrete in Language Modeling}

\author{%
  Yuxin Chen\thanks{Equal contribution}\\
  UIUC\\
  \texttt{nealchen@illinois.edu} \\
  \And
  Chumeng Liang\footnotemark[1]\\
  UIUC\\
  \texttt{chumengl@illinois.edu} \\
  \And
  Hangke Sui\footnotemark[1]\\
  UIUC\\
  \texttt{hangkes2@illinois.edu} \\
  \And
  Ruihan Guo\\
  UIUC\\
  \texttt{ruihang6@illinois.edu} \\
  \And
  Chaoran Cheng\\
  UIUC\\
  \texttt{chaoran7@illinois.edu} \\
  \And
  Jiaxuan You\\
  UIUC\\
  \texttt{jiaxuan@illinois.edu} \\
  \And
  Ge Liu\thanks{Corresponding author}\\
  UIUC\\
  \texttt{geliu@illinois.edu} \\
}

\begin{document}

\maketitle

\begin{abstract}
Continuous diffusion has served as a foundation for high-fidelity, controllable, and few-step generation across continuous data modalities such as images, videos, and molecular structures.
However, in language modeling, prior continuous diffusion language models (DLMs) lag behind discrete counterparts.
Existing categorical and simplex-based approaches operate over extremely large and sparse language spaces, while prior embedding-space approaches avoid this sparsity but lack a well-selected design space.
In this work, we close this gap with \textbf{LangFlow} by connecting embedding-space DLMs to Flow Matching, alongside three key innovations:
(1) we derive a novel ODE-based NLL bound for principled evaluation of continuous flow-based language models;
(2) we propose an information-uniform principle for setting the noise schedule, which motivates a learnable noise scheduler based on a Gumbel distribution; and
(3) we revise prior training protocols by incorporating self-conditioning, which improves both likelihood and sample quality for embedding-space DLMs and behaves differently from its use in discrete diffusion.
Putting everything together, LangFlow is competitive with top discrete DLMs on generative quality, reaching a generative perplexity of \textbf{92.3} on LM1B and \textbf{36.5} on OpenWebText. It also outperforms discrete counterparts on reasoning tasks like Sudoku. As the first continuous DLM shown to rival discrete diffusion in generative quality, LangFlow provides clear evidence that continuous diffusion is a promising paradigm for language modeling. Code: \url{https://github.com/nealchen2003/LangFlow}. 
\end{abstract}

\input{sections/intro}
\input{sections/prelim}
\input{sections/theory}
\input{sections/method}
\input{sections/experiments}
\input{sections/conclusion}

%%%%%%%%%%%%%%%%%%%%%%%%%%%%%%%%%%%%%%%%%%%%%%%%%%%%%%%%%%%%

\bibliography{iclr2025_conference}
\bibliographystyle{plainnat}

\onecolumn
\newpage
\appendix
% Tell cleveref that section/subsection labels after this point are appendices.
\crefalias{section}{appendix}
\crefname{appendix}{Appendix}{Appendices}
\Crefname{appendix}{Appendix}{Appendices}
\crefalias{subsection}{appendix}
\appendixpage
\addappheadtotoc
\input{appendix/related}
\input{appendix/algo.tex}

\input{appendix/proof.tex}
\input{appendix/design.tex}
\input{appendix/exp.tex}

\input{appendix/last.tex}

\newpage
\input{checklist.tex}

\end{document}

%% file: symbols.tex
\newcommand\ve{{\bm{e}}}

\newcommand\vp{{\bm{p}}}
\newcommand\vq{{\bm{q}}}

\newcommand\vu{{\bm{u}}}
\newcommand\vv{{\bm{v}}}

\newcommand\vx{{\bm{x}}}

\newcommand\vz{{\bm{z}}}

\newcommand\vtheta{{\bm{\theta}}}
\newcommand\vepsilon{{\bm{\epsilon}}}
\newcommand\R{\mathbb{R}}
\DeclareMathOperator{\mean}{\mathbb{E}}
\newcommand\N{{\mathcal{N}}}

\newcommand\mI{{\mathbf{I}}}
\newcommand\mE{{\mathbf{E}}}
\newcommand\T{^{\rm T}}
\newcommand\dif{\mathop{}\!\mathrm{d}}
\newcommand\Ls{{\mathcal{L}}}
\DeclareMathOperator{\sigmoid}{sigmoid}

\newcommand\e{{\mathrm{e}}}

\newcounter{note}

%% file: sections/intro.tex
\section{Introduction}

% background of diffusion
Diffusion models~\citep{sohl2015deep,song2019generative,ho2020denoising} have achieved remarkable success in generating data with continuous modalities, such as images~\citep{dhariwal2021diffusion,rombach2022high}, videos~\citep{ho2022video,blattmann2023stable}, molecular structures~\citep{watson2023novo,abramson2024accurate}, and robot behaviors~\citep{chi2025diffusion}.
A key reason behind this success is the flexibility of \emph{continuous diffusion}: it admits expressive latent trajectories, stable ODE/SDE-based sampling, and a mature toolbox of techniques such as self-conditioning, trajectory editing, and potential few-step acceleration via flow-based distillation.
These properties make continuous diffusion a powerful and versatile paradigm for generative modeling.
This success motivates extending diffusion models to categorical data, including graphs~\citep{vignac2022digress} and sequences~\citep{austin2021structured}, contributing to the broader goal of unifying cross-modality data generation under a single architecture~\citep{zhou2024transfusion,rojas2025diffuse}.

% why existing dlm does not perfectly fit
The most prominent milestone yet to be conquered for categorical diffusion is language modeling, where diffusion has morphed into two distinct formulations.
On one hand, discrete diffusion~\citep{austin2021structured,lou2023discrete,sahoo2024simple} reformulates diffusion processes directly in the categorical space based on the continuous-time Markov chain (CTMC) theory.
While achieving competitive performance, discrete diffusion sacrifices expressive latent spaces, which limits controllable and few-step generation.
On the other hand, simplex diffusion~\citep{cheng2024categorical,song2025shortlisting,cheng2025alpha} applies continuous diffusion over the probability simplex of language data.
Although theoretically sound and capable of preserving diffusion tricks, this approach suffers from the extreme sparsity of the simplex space, making score estimation particularly challenging.

% Why embedding-space DLM is worth doing
These shortcomings refocus our attention on embedding-space diffusion~\citep{li2022diffusion,dieleman2022continuous,gulrajani2023likelihood}, an overlooked family of diffusion language models that inherently avoid sparsity and offer editability.
Embedding-space diffusion generates token embeddings progressively from Gaussian noise with editable generation paths and denser data spaces.
This makes embedding-space diffusion a promising direction to explore for language modeling.

% Challenges of embedding-space dlm
However, exploration of embedding-space DLMs has long been blocked by several challenges.
First, the theoretical grounding of embedding-space DLMs remains unclear.
Training objectives in prior work are either heuristic or cumbersome to implement, for example requiring dynamically sliced batches to optimize different objectives~\citep{gulrajani2023likelihood}.
Second, a reliable ODE-based estimation of perplexity (PPL) -- the primary evaluation metric in language modeling -- has yet to be established for embedding-space DLMs, as prior works rely exclusively on SDE-based bounds.
Third, the absence of PPL evaluation further obscures the goal of optimizing training techniques, leaving good design choices unclear.
Without a solid theoretical foundation for evaluation and improved training techniques, the potential of embedding-space diffusion language models remains difficult to assess and improve.

% Our contribution & how to solve the challenges
In this paper, we observe that the ODE formulation of embedding-space DLMs naturally aligns with Flow Matching~\citep{lipman2023flow} (\cref{sec:theory.bregman-fm}). This theoretical connection allows us to derive a novel ODE-based upper bound of negative log-likelihood (NLL), which tackles the challenge of evaluating embedding-space DLMs and complements prior SDE-based bounds (\cref{sec:ode_sampling}).
Based on this framework, we further explore the training techniques of LangFlow.
First, we show that the noise scheduler for embedding-space diffusion should differ substantially from those used in image diffusion because of differences in data geometry.
We then propose the information-uniform principle to obtain a noise schedule suited to language (\cref{sec:noise_scheduler}).
Second, we reveal how the effects of self-conditioning differ between discrete and continuous diffusion and refine the training protocol of embedding-space DLMs (\cref{sec:self_conditioning}).
Both training techniques substantially improve the performance of LangFlow.

% contribution
\textbf{Contributions:} To summarize, our contributions are three-fold:

\textbf{1)} We rethink embedding-space diffusion language models through the lens of Flow Matching.
This unlocks key theoretical insights and empirical gains, including a theoretically grounded cross-entropy loss and a novel ODE-based negative log-likelihood (NLL) bound that provides an evaluation protocol aligned with deterministic ODE sampling for embedding-space DLMs.
The embedding space formulation circumvents the high-dimensionality bottlenecks in diffusion models on the probability simplex and one-hot space.

% \textbf{1)} We establish a theoretical connection between embedding-space diffusion and Flow Matching, from which we construct an exact upper bound of negative log likelihood for evaluating embedding-space diffusion language models.
\textbf{2)} We elucidate the design space of embedding-space DLMs and propose two key design choices that greatly improve training efficiency and sampling quality.
First, we introduce an information-uniform noise schedule with a new time-conditioning based on the logarithmic noise-to-signal ratio $\gamma$.
Our extensive profiling reveals a strong empirical fit to a Gumbel distribution over $\gamma$, an observation that differs sharply from common conclusions in image generation.
Second, we reveal the underexplored discrepancy of self-conditioning between its effect on discrete diffusion and that on continuous diffusion, rectifying the training recipe of continuous diffusion language modeling.
Both design choices substantially enhance the performance of embedding-space DLMs.

% \textbf{2)} We elucidate the effect of two key training techniques for embedding-space diffusion, information-uniform noise scheduling and self-conditioning, which improve the sample quality significantly.

\textbf{3)} Combining everything together, we introduce LangFlow, the first continuous DLM shown to rival discrete DLMs on both language modeling benchmarks and reasoning tasks. Specifically, LangFlow achieves a generative perplexity (Gen. PPL) of 92.3 on LM1B and 36.5 on OpenWebText, surpassing all discrete diffusion baselines. We also validate the clear advantage of LangFlow on Sudoku compared to discrete counterparts. LangFlow re-establishes the baseline for continuous diffusion language modeling, enabling future extensions to be built on a stronger foundation comparable to discrete diffusion.

% \paragraph{Contributions}
% We summarize our contributions as follows:
% \begin{itemize}
%     \item \textbf{Unified theoretical framework.} Theoretically, we connect embedding-space diffusion language models to variational flow matching, from which we build a unified framework to encompass all existing methods as its parameterizations.
%     \item \textbf{Design space of Emb-DLMs.} We systematically analyze the design choices of Emb-DLMs, providing substantial evidences of the design choices, especially in the training objective and the noise schedule, emphasizing the importance of designing modality-specific diffusion models.
%     \item \textbf{Empirical evidence for continuous diffusion in language modeling.} Empirically, we are the first ones to break the old-fashioned belief that continuous diffusion is significantly inferior to discrete diffusion in language modeling, suggesting the considerable potential of continuous diffusion language models.
% \end{itemize}
% }
\color{black}

%% file: sections/prelim.tex
\section{Preliminaries}

Flow Matching (FM)~\citep{lipman2023flow} is a generative modeling paradigm that learns a velocity field $\vu_t(\vz_t)$ to transport a simple prior $p_{\rm prior}$ (e.g., standard Gaussian) to the data distribution $p_{\rm data}$.
Starting from $\vz_0 \sim p_{\rm prior}$, solving the ordinary differential equation (ODE) $\dif\vz_t = \vu_t(\vz_t)\dif t$ yields $\vz_1 \sim p_{\rm data}$.

The marginal velocity field $\vu_t(\vz_t)$ is constructed by marginalizing conditional velocity fields.
Under mild regularity conditions, if a conditional velocity field $\vu_t(\vz_t \mid \vz)$ transforms $\vz_0 \sim p_{\rm prior}$ to a specific data point $\vz_1=\vz$, the marginal velocity field $\vu_t(\vz_t)$ generates $p_{\rm data}$:
\begin{equation}
\vu_t(\vz_t) = \mean[\vu_t(\vz_t \mid \vz) \mid \vz_t] = \int \vu_t(\vz_t \mid \vz) p(\vz \mid \vz_t) \dif\vz.
\end{equation}
Typical FM models utilize an affine Gaussian probability path as the conditional flow:
\begin{equation}\label{eq:fm}
    \vz_t = \alpha_t\vz + \sigma_t\vepsilon, \quad \vepsilon\sim p_{\rm prior},
\end{equation}
which corresponds to the conditional velocity field:
\begin{equation}
\vu_t(\vz_t \mid \vz) = \dot\sigma_t\frac{\vz_t-\alpha_t\vz}{\sigma_t} +\dot\alpha_t\vz,
\end{equation}
where $\alpha_t, \sigma_t$ are differentiable scalar schedules specifying the transformation, satisfying $\alpha_0=0,$ $\sigma_0=1$ and $\alpha_1=1, \sigma_1=0$, with time derivatives $\dot\alpha_t, \dot\sigma_t$.

In practice, computing the true marginal velocity field $\vu_t(\vz_t)$ is intractable.
Instead, it is approximated by a neural network $\vv_\vtheta(\vz_t,t)$, trained via conditional Flow Matching:
\small
\begin{align}
\Ls_{\rm FM}(\vtheta) & =\int_0^1 \mean\left[ \left\| \vv_\vtheta(\vz_t,t) - \vu_t(\vz_t) \right\|^2 \right] \dif t = \int_0^1 \mean\left[ \left\| \vv_\vtheta(\vz_t,t) - \vu_t(\vz_t \mid \vz) \right\|^2 \right] \dif t + \text{const},
\end{align}
\normalsize
where the constant term is independent of $\vtheta$ and thus omitted during optimization.

Since $\vu_t(\vz_t \mid \vz)$ is linear in $\vz$, the velocity network can be equivalently parameterized using a denoiser $\hat\vz_\vtheta(\vz_t,t) \approx \mean[\vz \mid \vz_t]$:
\begin{equation}
\vv_\vtheta(\vz_t, t) = \vu_t\big(\vz_t \mid \hat\vz_\vtheta(\vz_t,t)\big).
\label{eq:velocity_denoiser}
\end{equation}

Flow Matching is equivalent to learning the probability flow ODE of standard diffusion models~\citep{sohl2015deep,song2019generative,ho2020denoising} when $p_{\rm prior}=\N(\bm 0,\mI)$.
Since we consistently use this prior, we use the terms Flow Matching and diffusion models interchangeably.

\textbf{Related work.}
We provide a more detailed discussion of related discrete, continuous, and flow-map language diffusion models in Appendix~\ref{app:related-work}.

%% file: sections/theory.tex
\section{LangFlow: continuous language modeling via flow matching}
\label{sec:theory}

Existing methods for continuous diffusion language modeling typically train models by regressing denoised outputs, embeddings, or one-hot representations of the ground truth, while others rely on heuristic training objectives.
Furthermore, they lack a reliable ODE-based upper bound of the negative log-likelihood for evaluation.
In this section, we introduce \textbf{LangFlow}, a principled framework for continuous diffusion language models.
For training, we propose optimizing the cross-entropy loss and establish a connection between this objective and Flow Matching~\citep{lipman2023flow} through Bregman divergence.
We also derive a novel ODE-based upper bound of the negative log-likelihood for reliable evaluation.
Figure~\ref{fig:mainfig} provides an overview of the proposed pipeline, including both training and sampling procedures.

\subsection{Training flow matching on language}
\label{sec:theory.bregman-fm}
We begin with embedding-space diffusion for continuous diffusion language modeling.
Let $\mE \in \R^{V \times D}$ be an embedding matrix over a vocabulary of size $V$, mapping each text token $x^{(i)}$ to a $D$-dimensional vector $\ve_{x^{(i)}}$.
A token sequence $\vx = (x^{(1)}, \ldots, x^{(L)})$ is embedded as $\vz = (\ve_{x^{(1)}}, \ldots, \ve_{x^{(L)}}) \in \R^{L \times D}$, which serves as the generative target.

Flow Matching (FM)~\citep{lipman2023flow} defines a velocity field $\vu_t(\vz_t)$ that transports a Gaussian prior $p_\text{prior} = \N(\bm{0}, \mI)$ to the data distribution $p_\text{data}$ through the ODE $\dif\vz_t = \vu_t(\vz_t)\, \dif t$.
The standard formulation uses time $t \in [0,1]$ and an affine probability path with scalar schedules $\alpha_t$ and $\sigma_t$:
\begin{equation}
  \vz_t = \alpha_t \vz + \sigma_t \vepsilon, \quad \vepsilon \sim \N(\bm{0}, \mI),
\end{equation}
where $\alpha_t$ grows from $0 \to 1$ and $\sigma_t$ decays from $1 \to 0$ as $t$ goes from $0$ to $1$.
While natural, this formulation ties both the objective and the dynamics to a specific choice of time schedule.

\textbf{Time conditioning through $\gamma$-path.} 
To reduce this schedule dependence, we reparameterize the flow path using the logarithmic noise-to-signal ratio (logNSR) $\gamma$.
The motivation is that denoising difficulty is primarily controlled by noise level rather than by an arbitrary time index; this is also consistent with Variational Diffusion Models~\citep{kingma2021variational}.
For any valid path of the form $\vz_t=\alpha_t\vz+\sigma_t\vepsilon$, we define $\gamma_t = \log(\sigma_t^2/\alpha_t^2).$
In typical settings, $\dot\alpha_t>0$ and $\dot\sigma_t<0$, so $\gamma_t$ is strictly monotone decreasing.
This induces a bijection between $t \in [0,1]$ and $\gamma \in \mathbb{R}\cup\{\pm\infty\}$, effectively reparameterizing the flow in terms of $\gamma$ instead of $t$.
Since $\alpha_0=0$ and $\sigma_1=0$, the endpoints correspond to pure noise at $t=0$ ($\gamma=+\infty$) and clean data at $t=1$ ($\gamma=-\infty$).

Under this change of variables, the marginal distributions are preserved, and the training objective is equivalent up to the Jacobian factor $|\dot \gamma_t|$.
Restricting attention to variance-preserving (VP) paths satisfying $\alpha_\gamma^2+\sigma_\gamma^2=1$, we can rewrite the path directly in terms of $\gamma$:
\begin{align}\label{eq:langflow-path}
    \vz_\gamma = \alpha_\gamma \vz + \sigma_\gamma \vepsilon, 
    \qquad \vepsilon \sim \N(\bm{0}, \mI), \qquad \gamma \in \mathbb{R}\cup\{\pm\infty\},
\end{align}
with
\begin{align}
    \sigma_\gamma^2 = \sigmoid(\gamma), 
    \qquad
    \alpha_\gamma^2 = 1-\sigma_\gamma^2 = \sigmoid(-\gamma).
\end{align}
We refer to this flow simply as the $\gamma$-path.
This unification inherently encompasses standard diffusion trajectories and makes explicit that learning is anchored to the noise level $\gamma$, rather than to a particular time parameterization.
We base our implementation on the $\gamma$-path and use $\gamma$ as our time conditioning variable instead of $t$; further motivation is provided in \cref{sec:noise_scheduler}.

\textbf{Training flow matching via cross entropy.}
We now propose our cross-entropy training objective for Flow Matching on categorical data like language, which is closely related to Variational Flow Matching (VFM)~\citep{eijkelboom2024variational} (Appendix~\ref{app:vfm}) and is a special case of Bregman divergence minimization~\citep{guzman2025exponential} (Appendix~\ref{app:bregman-divergence}).

Let the model output be $\hat\vx_\vtheta(\vz_\gamma,\gamma)$, whose $(i,k)$ entry approximates $\Pr(x^{(i)}=k\mid \vz_\gamma)$.
Let $\bm1_x$ denote the one-hot vector of token $x$.
Let $\pi(\gamma)$ be a noise sampling distribution.
We define the training objective along the $\gamma$-path as
\begin{equation}
\label{eq:ce_loss}
  \Ls_{\rm CE}(\vtheta) = \mean_{\gamma \sim \pi} [ \ell_{\rm CE}(\gamma) ], % \\
% \label{eq:ce_loss_gamma}
~\text{where}~
  \ell_{\rm CE}(\gamma) = \mean\left[-\frac{1}{L}\sum_{i=1}^L \log \hat x_\vtheta^{(i, x^{(i)})}(\vz_\gamma, \gamma)\right].
\end{equation}
Note that $\Ls_{\rm CE}$ is equivalent to KL divergence up to an additive constant:
\begin{equation}
\label{eq:ce_is_kl_plus_entropy}
  \mean\Bigl[-\log \hat x_\vtheta^{(i, x^{(i)})}(\vz_\gamma, \gamma) \Bigm| \vz_\gamma \Bigr] = \mathcal{D}_{\rm KL}\bigl(P(x^{(i)} \mid \vz_\gamma) \bigm\| \hat\vx_\vtheta^{(i)}(\vz_\gamma, \gamma) \bigr) + \mathcal{H}\bigl( x^{(i)} \bigm| \vz_\gamma \bigr),
\end{equation}
where $P(x^{(i)} \mid \vz_\gamma)$ denotes the posterior distribution of $x^{(i)}$ given $\vz_\gamma$, and $\mathcal{H}\bigl( x^{(i)} \bigm| \vz_\gamma \bigr)$ denotes the posterior entropy.
Therefore, for each token position, the unique optimum of this factorized prediction objective is $\hat\vx_\vtheta^{(i)}=P(x^{(i)} \mid \vz_\gamma)$.

Given predicted token probabilities $\hat x_\vtheta^{(i)}$, we recover a continuous denoiser by taking the embedding expectation of $\hat x_\vtheta^{(i)}$ over the vocabulary:
\begin{equation}
    \hat\vz_\vtheta^{(i)}(\vz_\gamma,\gamma)
    =
    \sum_{k=1}^V \hat x_\vtheta^{(i,k)}(\vz_\gamma,\gamma)\,\ve_k
    =
    \mE^\top \hat\vx_\vtheta^{(i)}(\vz_\gamma,\gamma).
    \label{eq:denoiser-synth}
\end{equation}
This establishes a direct coupling between discrete likelihood training and continuous flow estimation: the model is optimized in token space via CE, while the associated continuous denoiser is obtained deterministically from token probabilities for ODE-based sampling.
\begin{figure}[t]
  \centering
  \includegraphics[width=0.95\linewidth]{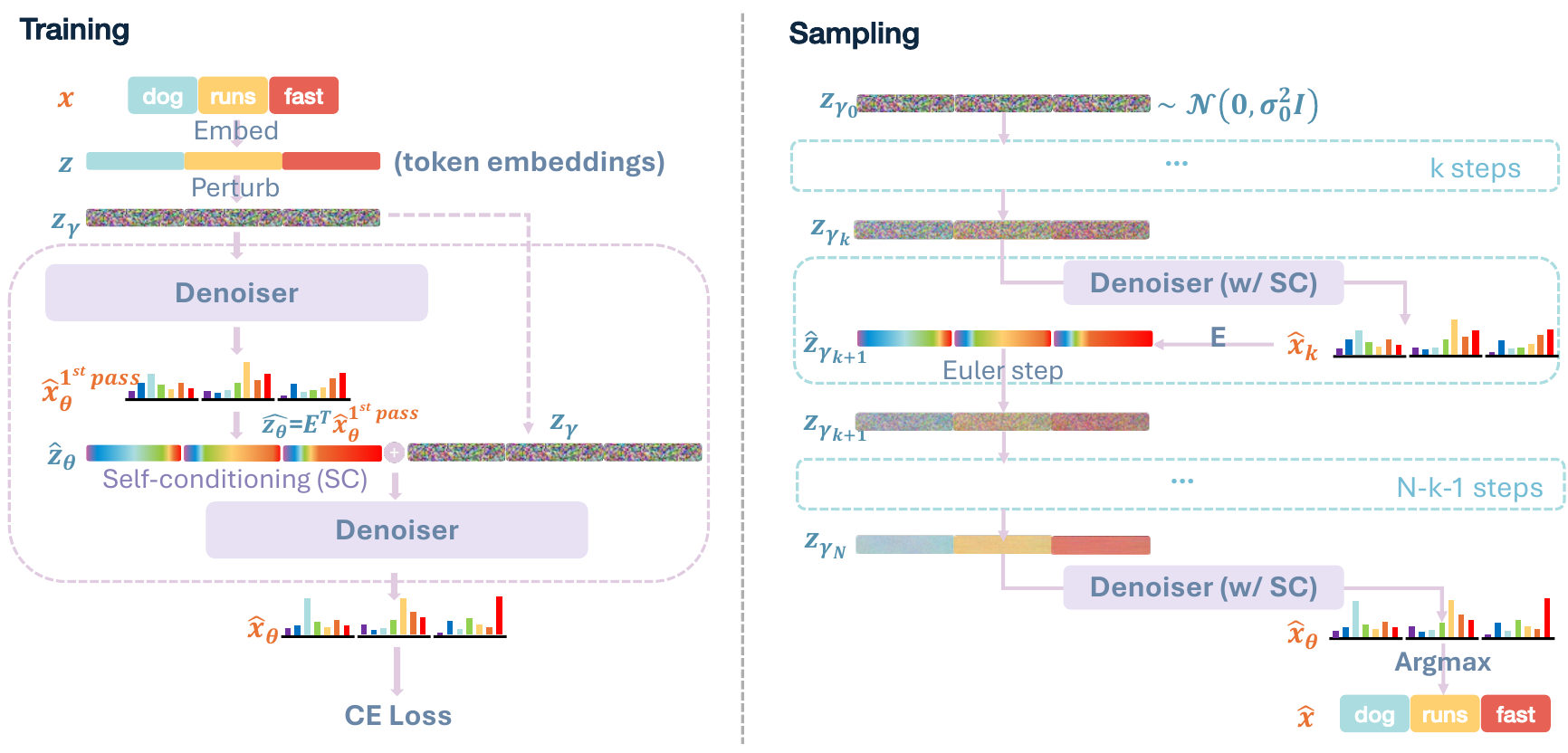}
  \caption{\textbf{LangFlow takes noisy token embeddings as inputs and predicts clean probabilities.} In the \textbf{training (\textit{left})}, we first map discrete token sequence $\vx$ to token embeddings $\vz$ via a learnable embedding matrix $\mE$.
  Then, we perturb $\vz$ into noisy embeddings $\vz_\gamma$ according to $\gamma$, the sampled logarithmic noise-to-signal ratio.
  Finally, we use a denoiser network to predict the categorical distribution $\hat\vx_\vtheta(\vz_\gamma, \gamma)$ and supervise this prediction with the cross-entropy loss.
  Self-conditioning is randomly activated.
  In the \textbf{sampling (\textit{right})}, we start from Gaussian noise $\vz_{\gamma_0}$ and iteratively denoise it towards clean token embeddings $\vz_{\gamma_N}$.
  We define the moving target $\hat{\vz}_\vtheta$ for every step by embedding predicted clean probabilities $\hat\vx_\vtheta$ with $\mE$.
  The final token sequence is obtained by $\arg\max$ decoding of $\hat\vx_\vtheta(\vz_{\gamma_N}, \gamma_N)$.
}
  \label{fig:mainfig}
\end{figure}

\textbf{Training pipeline.} Figure~\ref{fig:mainfig} (Left) summarizes our training pipeline.
We first embed language tokens $\vx$ with a learnable embedding matrix $\mE$.
After that, we perturb the clean embeddings $\vz$ to the noisy embeddings $\vz_\gamma$.
Our model is then applied to predict the token probabilities $\hat \vx_\vtheta$ from $\vz_\gamma$.
We use the CE loss in \cref{eq:ce_loss} to train the model.
The complete training procedure of LangFlow is summarized in Algorithm~\ref{alg:training} (deferred to Appendix~\ref{app:algorithm}).

Under this pipeline, our model $\vtheta$ is fed with (noisy) embeddings $\vz_\gamma\in\mathbb{R}^{L\times D}$ and predicts token probabilities $\hat \vx_\vtheta\in\mathbb{R}^{L\times V}$, similar to discrete diffusion and autoregressive language models.
This means that we can share the same network architecture with discrete DLMs~\citep{sahoo2024simple,sahoo2025diffusion}, which is a modified version of Diffusion Transformers~\citep{peebles2023scalable}.

\subsection{ODE sampling and the bound for negative log likelihood}
\label{sec:ode_sampling}

Continuous diffusion models typically offer two sampling paradigms: Stochastic Differential Equations (SDEs) or Probability Flow Ordinary Differential Equations (PF-ODEs)~\citep{song2021scorebased}.

\textbf{Why ODE over SDE?}
Although SDEs are common in diffusion frameworks, their continuous noise injection destroys the deterministic bijective mapping between the prior and data distributions.
This stochasticity inherently resists flow-based distillation techniques like Consistency Models~\citep{song2023consistency}.
By exclusively using the deterministic ODE, LangFlow preserves this crucial bijection, enabling future acceleration into efficient few-step generators.

\textbf{Sampling pipeline.}
Figure~\ref{fig:mainfig} (Right) illustrates the pipeline of our ODE sampling.
In practice, we solve the ODE over a finite $\gamma$ range $[a, b]$ rather than the whole extended real line $\mathbb{R}\cup\{\pm\infty\}$.
Given $N$ sampling steps, we set the starting point $\gamma_0=b$ and the end point $\gamma_N=a$ and distribute $\gamma_k$ within $[a, b]$ according to a pre-defined discretization (see \cref{sec:noise_scheduler}).
The initial state $\vz_{\gamma_0}$ is sampled from $\mathcal{N}(\bm{0}, \sigma_b^2\mI)$.
At each sampling step, we run the model forward to obtain the token probabilities $\hat \vx_k$, embed $\hat \vx_k$ to derive the denoised embedding $\hat \vz_k$, and use $\hat \vz_k$ to determine the solver direction.
The solver then produces $\vz_{k+1}$ as the input to the next sampling step.
After integrating the ODE down to $\gamma_N=a$ by the above iteration, we take a final token prediction $\vx^{(i)}=\arg\max \hat\vx_{\vtheta}^{(i)}(\cdot \mid \vz_a, a)$ with an extra model forward.
In our implementation, we follow existing literature and use a first-order solver along the $\gamma$-path, detailed in \cref{app:exp:solver}.
Our ODE sampling procedure is outlined in Algorithm~\ref{alg:sampling}.

\textbf{Perplexity estimation.}
Perplexity (PPL) remains the primary metric for evaluating data likelihood in language models.
Previous methods~\citep{gulrajani2023likelihood} compute PPL using stochastic bounds (NELBO).
Based on the continuous Flow Matching formulation, we can estimate NLL in our framework by integrating along the ODE path.
Specifically, by slightly modifying the final step from taking $\arg\max$ to sampling $\vx^{(i)} \sim \hat\vx_\vtheta^{(i)}(\cdot \mid z_a, a)$, %New
the following theorem gives a novel ODE-based upper bound for estimating perplexity under deterministic ODE sampling.

\begin{restatable}{theorem}{odePplBound}
\label{thm:ode_ppl_bound}
Assume the ODE solution on $[a,b]$ is well-defined and invertible, with an integrable divergence term. For a sequence $\vx = (x^{(1)}, \dots, x^{(L)})$ of length $L$ and embedding dimension $D$, the log-likelihood of LangFlow has the following evidence lower bound:
\begin{equation}
\log P_\vtheta(\vx) \ge \mathbb{E}_{q_a(\vz_a\mid \vx)} \Bigg[ \frac{LD}{2} - \frac{\|\vz_b\|^2}{2\sigma_b^2} + \sum_{i=1}^L \log \hat\vx_{\vtheta}^{(i, x^{(i)})}(\vz_a, a) - \int_a^b \frac{\alpha_\gamma}{2} \nabla \cdot \hat{\vz}_\vtheta(\vz_\gamma, \gamma) \dif \gamma \Bigg],
\end{equation}
where $q_a(\vz_a\mid\vx)=\mathcal{N}(\alpha_a \mE\T \vx, \sigma_a^2 \mI)$, $\vz_\gamma (a \leq \gamma \leq b)$ is the reverse ODE trajectory given $\vz_a$, and $\nabla \cdot \hat{\vz}_\vtheta$ represents the divergence w.r.t.\ $\vz_\gamma$.
\end{restatable}
\begin{proof}
The proof is deferred to \cref{app:proof_ode_ppl}.
\end{proof}

%% file: sections/method.tex
\section{Improved design choices of continuous diffusion language models} \label{sec:4}
We next identify two key design choices for training continuous diffusion language models like LangFlow: noise scheduling and self-conditioning.
In \cref{sec:noise_scheduler}, we show that language exhibits a fundamentally different noise geometry from the common practice of diffusion in the image and video domains. In \cref{sec:self_conditioning}, we show that self-conditioning strengthens continuous diffusion differently from discrete diffusion. Putting the two choices together allows LangFlow to match the sample quality of discrete diffusion.

% Building upon the tokenwise CE objective established in \cref{sec:theory}, \cref{sec:noise_scheduler} discusses the noise scheduling of DLMs, highlighting a fundamental difference between our formulation and standard continuous generative models.
% We propose aligning the noise schedule with the information gain trajectory, which enables LangFlow to achieve performance on par with the SOTA discrete diffusion models (\cref{sec:5}).
% In \cref{sec:self_conditioning}, we identify self-conditioning as an indispensible part of DLMs.
% Finally, in \cref{sec:ode_sampling}, we highlight the deterministic sampling strategy of LangFlow and provide an evaluation protocol compatible with ODE, as opposed to existing DLMs evaluated through stochastic sampling. Additional design choices are discussed in Appendix~\ref{app:design-choices}.

\subsection{Noise scheduler}
\label{sec:noise_scheduler}
Training continuous diffusion requires choosing how to schedule the noise level. We first follow strong diffusion baselines such as Stable Diffusion 3~\citep{esser2024scaling} and EDM~\citep{karras2022elucidating} and use $t$ as our time conditioning, with uniform noise scheduling by default. We plot our cross-entropy loss w.r.t.\ $t$ at 50k training steps over the entire LM1B dataset~\citep{chelba2013one} to determine which region of $t$ should receive more training steps.

Figure~\ref{fig:loss_gumbel} (Left) visualizes the result. We observe that our loss is nearly zero for $t\in[0.2,1.0]$. This suggests that \textbf{1)} during training, the model can already predict the correct token for $t>0.2$, and \textbf{2)} during sampling, the model contributes little new information through $t\in[0.2,1.0]$. Therefore, uniform noise schedulers based on $t$ allocate more than half of their training and sampling steps to noise levels that carry little useful information. This is intuitive because generation paths toward categorical data have finite and isolated destinations, which can remain distinguishable even at high noise levels. The noise schedule should therefore depend strongly on the geometry of the underlying data distribution, so we cannot simply transfer image-diffusion schedules to this setting.

To optimize the noise scheduler of LangFlow, we use the following techniques:

\textbf{Time conditioning on $\gamma$-path.} We observe from Figure~\ref{fig:loss_gumbel} (Left) that most loss reduction occurs in the very noisy region where the signal-to-noise ratio (SNR) is below 0.0625. Hence, we use time conditioning based on $\gamma_t = \log(\sigma_t^2/\alpha_t^2)$, the logarithmic noise-to-signal ratio (NSR), in \cref{sec:theory.bregman-fm}. The main advantage is that the time conditioning shifts linearly when NSR scales exponentially, expanding the original noisy region $t\in[0,0.2]$ to a broader interval. This allows the network to better utilize the resolution of time conditioning and makes the loss profile easier to capture.

\textbf{Information-uniform principle in noise scheduling.}
We observe that the average CE loss $\ell_{\rm CE}(\gamma)$ at noise level $\gamma$ remains stable across training stages (\cref{fig:loss_gumbel}, middle), regardless of the noise scheduler. Using the decomposition in \cref{eq:ce_is_kl_plus_entropy}, the KL term quickly approaches 0 early in training, leaving $\ell_{\rm CE}(\gamma)$ dominated by the irreducible posterior entropy $H_\gamma = \frac{1}{L} \sum_{i=1}^L H(x^{(i)} \mid \vz_\gamma)$. This explains why the loss values remain stable after the early training stage: the main component of the loss is a property of the data and cannot be reduced by optimization.
% \notebox{%
%   At each noise level $\gamma$, training loss $\Ls_\gamma$ rapidly converges to mean tokenwise posterior entropy $H_\gamma$.
% }
% \begin{tcolorbox}%[highlightcompact]
% \textbf{Note 1}: At noise level $\gamma$, $\Ls_\gamma$ quickly converges to tokenwise posterior entropy $H_\gamma$ of data.
% \end{tcolorbox}

% The derivative $H'_\gamma$ measures information gained about tokens per unit increase of $\gamma$. Therefore, we should allocate training resources and generative sampling steps to regions where the model acquires the most information (i.e., where $H'_\gamma$ is concentrated).
Diffusion sampling can then be described as a process that progressively gains information, i.e., reduces entropy $H_\gamma$ to zero, where the derivative $H'_\gamma=\frac{\dif H_\gamma}{\dif \gamma}$ measures the rate at which information about the clean tokens is gained per unit change in $\gamma$. We therefore propose a noise scheduling principle that allocates information gain uniformly over the noise density, so that each sampling step receives an equal amount of information. This suggests that both training effort and sampling steps should be concentrated in regions where $H'_\gamma$ is large.

% \notebox{%
%   Match time scheduling to information gain per unit noise level, i.e., $H'_\gamma$.
% }

\begin{figure*}[htbp]
  \centering

  \begin{subfigure}[t]{0.32\textwidth}
    \centering
    \includegraphics[width=\linewidth]{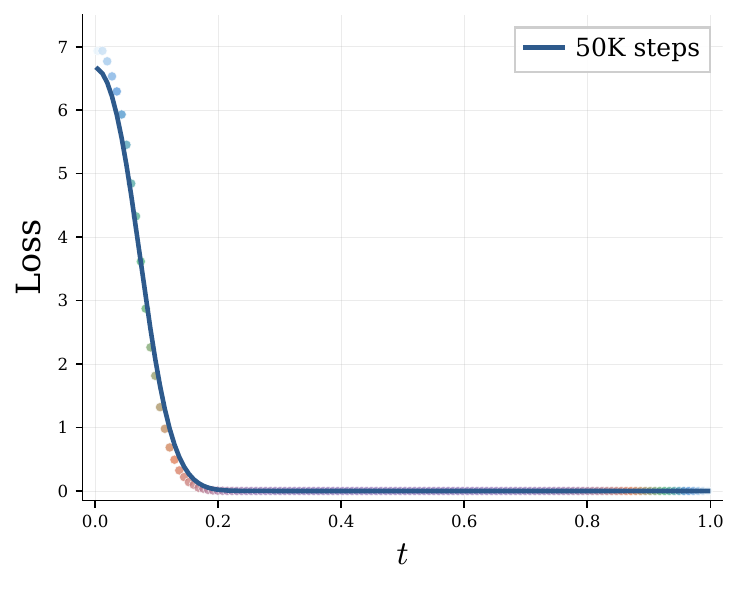}
  \end{subfigure}
  \hfill
  \begin{subfigure}[t]{0.32\textwidth}
    \centering
    \includegraphics[width=\linewidth]{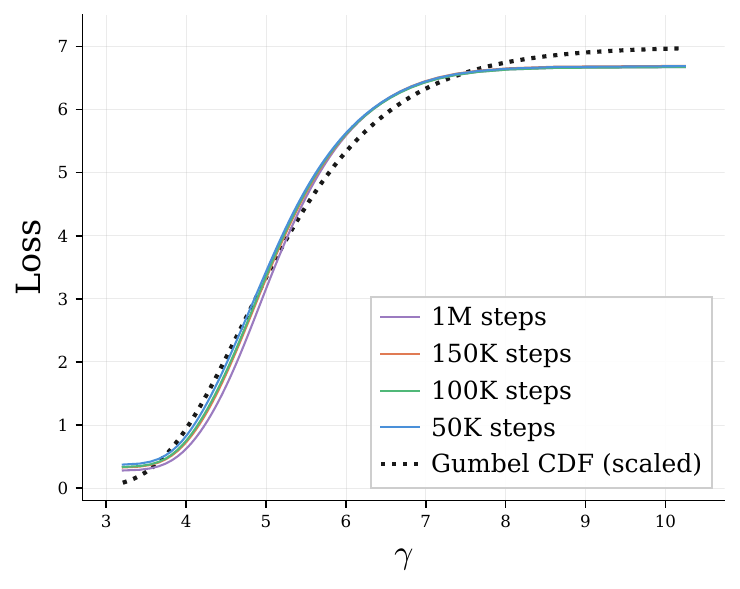}
  \end{subfigure}
  \hfill
  \begin{subfigure}[t]{0.32\textwidth}
    \centering
    \includegraphics[width=\linewidth]{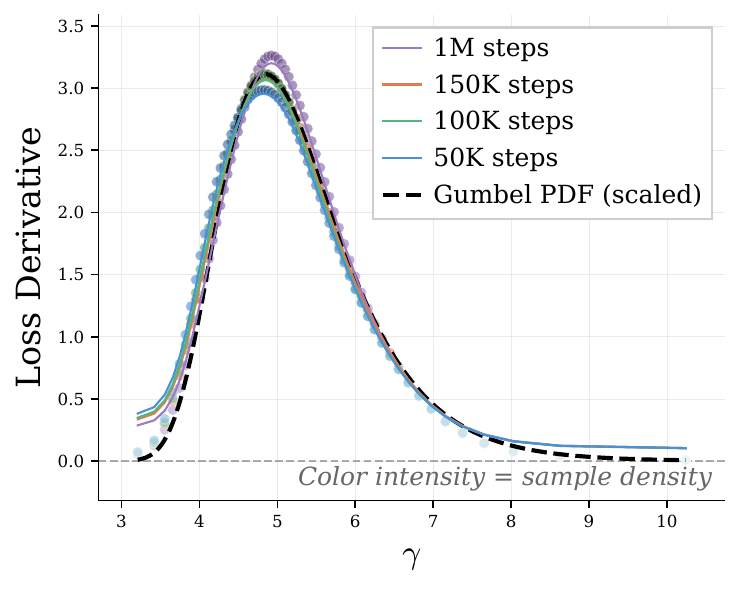}
  \end{subfigure}

  \caption{\textbf{Loss geometry.} (Left) Training loss as a function of the OT flow matching time $t=\sigmoid(-\gamma/2)$. (Middle) Loss over $\gamma=\log \mathrm{NSR}$ reveals a consistent geometry across training stages. (Right) The derivative $\partial \mathcal{L}/\partial \gamma$ concentrates in a narrow region, forming a stable structure.}
  \label{fig:loss_gumbel}
\end{figure*}

\paragraph{Scheduling the noise by the Gumbel distribution of $\gamma$.}
We follow the above principle to define a noise distribution whose density $\pi(\gamma) \propto H'_\gamma$. We use the empirical mean loss across checkpoints at different training steps as a surrogate for $H_\gamma$ and apply finite-difference smoothing. The curve of $H'_\gamma$ is positively skewed (\cref{fig:loss_gumbel}, right) and is well matched by a Gumbel distribution:
\begin{equation}
  H_\gamma = H_{+\infty} \cdot \exp\left(-\exp\left(-\frac{\gamma - P_\mu}{P_\beta}\right)\right),
\end{equation}
where $H_{+\infty}$ is the average posterior entropy at pure noise and $P_\mu,P_\beta$ are the location and scale parameters of the Gumbel distribution. We observe that the fitted Gumbel distribution drifts subtly as training evolves. Hence, we make $H_{+\infty}$, $P_\mu$, and $P_\beta$ learnable parameters, trained by a scheduler loss $\Ls_{\rm Scheduler} = \mean[(\ell_{\rm CE}(\gamma) - H_\gamma)^2]$. We use $\Ls_{\rm Scheduler}$ to train the noise scheduler and $\Ls_{\rm CE}$ to train the main model, with stop-gradient operations separating their updates.

\textbf{Implementation.} In practice, we clip the range of $\gamma$ to $[a,b]$, where $a$ and $b$ are the lower and upper $10^{-5}$ quantiles of the Gumbel distribution, respectively. During training, we sample $\gamma$ from the Gumbel distribution. During sampling with $N$ steps, we set the intermediate schedule to the $i/N$ quantiles of the Gumbel distribution for $i=1, \dots, N-1$. This choice significantly improves the sample quality of LangFlow, reducing the generative perplexity from $\sim$1000 to 154.2.

% \notebox{
% Training continuous generative models on disparate data modalities requires scheduling adapted to their inherent information curves.
% }

\subsection{Self-conditioning}
\label{sec:self_conditioning}
Self-conditioning~\citep{chen2022analog} is an established technique in diffusion models and is generally known to improve generation quality and perplexity (Gen.\,PPL and PPL, defined in \cref{sec:5}). Specifically, it feeds the previous step's prediction $\hat\vz_\vtheta(\vz_\gamma, \gamma)$ back as an auxiliary input to the network. During training, self-conditioning is applied randomly with a certain probability by feeding the model its own prediction for one round, so the model learns to exploit the prior prediction when available. During sampling, self-conditioning is always enabled: it is initialized to zero at the first step and set to the previous prediction at subsequent steps. Figure~\ref{fig:mainfig} shows how self-conditioning is incorporated into training and sampling.

Despite these general benefits, prior works on discrete diffusion language models typically evaluate PPL with self-conditioning disabled. This evaluation protocol has been carried over to continuous diffusion. We empirically find that applying this protocol to continuous diffusion is inappropriate. As shown by the ablation in Table~\ref{tab:sc}, turning self-conditioning on and off for MDLM (discrete) and LangFlow (our continuous model) reveals an asymmetry: for MDLM, self-conditioning improves Gen.\,PPL but degrades PPL, which helps explain its omission in discrete PPL evaluations. For LangFlow, the picture is different: self-conditioning significantly improves both Gen.\,PPL and PPL, which is enough to close the PPL gap between continuous and discrete diffusion in our setting.

How the effects of self-conditioning differ across formulations remains an open question; we provide observations in \cref{app:self-cond}. Nevertheless, in our experiments, self-conditioning is a crucial component for LangFlow to match the sample quality of discrete diffusion.

%% file: sections/experiments.tex
\section{Experiments}
\label{sec:5}

\textbf{Datasets.}
We evaluate LangFlow on two standard language modeling benchmarks: LM1B~\citep{chelba2013one} and OpenWebText (OWT) with sequence packing~\citep{raffel2020exploring}, following the existing literature on DLMs. Following \citet{sahoo2024simple,sahoo2025diffusion}, we reserve the last 100k documents of OWT as the validation set.

\begin{table*}[t]
\centering
\vspace{-2mm}
\caption{
Performance comparison on LM1B and OWT. We report upper bounds on perplexity (PPL) for DLMs and generative perplexities (Gen.\,PPL) based on GPT-2~\citep{radford2019language}. $^{\ddagger}$ denotes results from retrained baselines. We use \colorbox{gold!60}{gold}, \colorbox{silver!60}{silver}, and \colorbox{bronze!60}{bronze} to denote the 1st, 2nd, and 3rd place among DLMs, respectively.
}
\label{tab:diffusion_lm_comparison}
\small
\setlength{\tabcolsep}{5pt}
\renewcommand{\arraystretch}{1.08}
\begin{tabular}{llcccc}
\toprule
\multirow{2}{*}{\textbf{Category}} 
& \multirow{2}{*}{\textbf{Model}} 
& \multicolumn{2}{c}{\textbf{LM1B}} 
& \multicolumn{2}{c}{\textbf{OWT}} \\
\cmidrule(lr){3-4} \cmidrule(lr){5-6}
& 
& \textbf{Gen.\,PPL}$\downarrow$ 
& \textbf{PPL}$\downarrow$ 
& \textbf{Gen.\,PPL}$\downarrow$ 
& \textbf{PPL}$\downarrow$ \\
\midrule

Autoregressive
& Transformer 
& 66.7$^{\ddagger}$ 
& 22.8$^{\ddagger}$ 
& 35.9 
& 17.5 \\

\midrule
\multirow{4}{*}{\shortstack[l]{Discrete Diffusion\\(Absorbing)}}
& D3PM Absorb 
& -- & 76.9 
& -- & -- \\
& DiffusionBert 
& -- & 63.8 
& -- & -- \\
& SEDD Absorb 
& 115.9$^{\ddagger}$ 
& \cellcolor{bronze!60}32.0$^{\ddagger}$ 
& 105.0 
& \cellcolor{silver!60}24.1 \\
& MDLM 
& 103.9$^{\ddagger}$ 
& \cellcolor{silver!60}31.0$^{\ddagger}$ 
& 104.9 
& \cellcolor{gold!60}23.2 \\

\midrule
\multirow{4}{*}{\shortstack[l]{Discrete Diffusion\\(Uniform)}}
& D3PM Uniform 
& -- & 137.9 
& -- & -- \\
& SEDD Uniform 
& -- 
& 40.3 
& 99.9 
& 29.7 \\
& UDLM 
& 99.8$^{\ddagger}$ 
& 33.8$^{\ddagger}$ 
& -- 
& 27.4 \\
& Duo 
& 97.6$^{\ddagger}$ 
& 33.6$^{\ddagger}$ 
& 77.7 
& 25.2 \\
\midrule

\multirow{6}{*}{\shortstack[l]{Continuous Diffusion}}
& Diffusion-LM 
& -- & 118.6 
& -- & -- \\
& Plaid 
& \cellcolor{silver!60}77.3$^{\ddagger}$ 
& 32.4$^{\ddagger}$ 
& -- 
& -- \\
& FLM (1024 steps) 
& 96.9 
& -- 
& \cellcolor{bronze!60}62.2 
& -- \\
& \textbf{LangFlow (ODE w/o chain rule)}
& \cellcolor{bronze!60}\textbf{92.3}
& \cellcolor{gold!60}\textbf{30.0}
& \cellcolor{silver!60}\textbf{36.5}
& \cellcolor{bronze!60}\textbf{24.6} \\
& \textbf{LangFlow (ODE w/ chain rule)}
& \cellcolor{bronze!60}\textbf{92.3}
& \textbf{45.3}
& \cellcolor{silver!60}\textbf{36.5}
& \textbf{33.3} \\
& \textbf{LangFlow (SDE)}
& \cellcolor{gold!60}\textbf{59.4}
& \textbf{44.2}
& \cellcolor{gold!60}\textbf{23.8}
& \textbf{32.2} \\

\bottomrule
\end{tabular}
\vspace{-4mm}
\end{table*}

\textbf{Training.} Following discrete diffusion baselines~\citep{sahoo2024simple,schiff2024simple,sahoo2025diffusion}, we use a DiT-style Transformer~\citep{peebles2023scalable} with rotary positional encoding~\citep{su2024roformer}. The model consists of 12 layers, a hidden size of 768, and 12 attention heads (130M parameters). The model is conditioned on $\gamma$ rather than $t$ and uses a learned Gumbel scheduler to match the empirical information-gain profile (Sec.~\ref{sec:noise_scheduler}). During training, $\gamma$ is sampled from the learned distribution, and the model is optimized with token-level cross-entropy together with the schedule fitting loss. Self-conditioning is enabled with probability 0.25. We also use a preconditioning skip connection with a 5K-step warmup. Following existing literature~\citep{sahoo2024simple,sahoo2025diffusion}, we train models for 1M steps on LM1B and OWT with a batch size of 512. We use a context length of 128 and the \texttt{bert-base-uncased} tokenizer~\citep{devlin2019bert} for LM1B, and a context length of 1024 and the \texttt{gpt2-large} tokenizer~\citep{radford2019language} for OWT.

\textbf{Evaluation.} We evaluate the model using generative perplexity (Gen.\,PPL) and perplexity on the validation dataset (PPL). For Gen.\,PPL, we follow the protocol in existing works~\citep{sahoo2024simple,schiff2024simple} to generate 1024 samples and compute their mean perplexity by GPT2-Large~\citep{radford2019language}. For all methods, samples are generated using 128 sampling steps and a sequence length of 128 on LM1B, and 1024 sampling steps with a sequence length of 1024 on OWT. The only exception is FLM, for which the numbers are taken from its original paper and are based on 1024 sampling steps. For PPL, we report available upper bounds for DLM baselines. 

Notably, diffusion models such as LangFlow admit both ODE and SDE formulations~\citep{song2021scorebased}. We therefore evaluate the trained LangFlow model under both frameworks. For the ODE evaluation, Theorem \ref{thm:ode_ppl_bound} establishes an upper bound of PPL for the ODE sampling process. We consider two PPL estimators that differ in whether to detach the self-conditioning in the integral, denoted by \textit{w/o chain rule} and \textit{w/ chain rule}~\citep{yang2026continuous}. The \textit{w/o chain rule} protocol is used for zero-shot transfer. For the SDE evaluation, we follow the same setup of our Plaid baseline~\citep{gulrajani2023likelihood} to report the \textit{SDE}-based PPL bound and use a sampling temperature of 0.9.

\textbf{Baselines.}
We compare LangFlow against the autoregressive Transformer and a wide range of discrete and continuous DLMs, including D3PM~\citep{austin2021structured}, DiffusionBert~\citep{he2023diffusionbert}, SEDD~\citep{lou2023discrete}, MDLM~\citep{sahoo2024simple}, UDLM~\citep{schiff2024simple}, and Duo~\citep{sahoo2025diffusion}. 
We also include continuous baselines such as Diffusion-LM~\citep{li2022diffusion} and Plaid~\citep{gulrajani2023likelihood}. Due to the lack of checkpoints on LM1B, we retrained AR, SEDD, MDLM, UDLM, Duo, and Plaid on LM1B for 1M steps, as these baselines represent strong recent methods. For OWT, we use checkpoints provided by the Duo codebase. We retrained baselines with their original network architecture, which is the same as ours except for Plaid~\citep{gulrajani2023likelihood}; for Plaid, we use its own codebase and Transformer architecture for training. This is because Plaid performs worse when using the DiT architecture used by LangFlow and the other baselines. To make the comparison fair, we use the same number of Transformer layers, heads, and hidden dimensions as LangFlow, keeping its number of parameters close to ours. Detailed implementation of Plaid is deferred to \cref{app:plaid}. For DiffusionBert, Diffusion-LM, and D3PM, we report the numbers from the original Duo paper~\citep{sahoo2025diffusion}. We also include FLM~\citep{lee2026one}, a concurrent continuous DLM, in our baselines. Since the FLM paper reports Gen.\,PPL but does not report PPL or specify a likelihood-evaluation protocol, we report the Gen.\,PPL from its original paper with the same (OWT) or more (LM1B) sampling steps compared to our setting.

\textbf{Language modeling.} Table~\ref{tab:diffusion_lm_comparison} shows that LangFlow achieves the best Gen. PPL on LM1B and OWT among DLMs. We also report the Gen. PPL under different NFEs and the frontier of Gen. PPL and data entropy in Appendix~\ref{app:gen_ppl}. Note that the stochastic sampling (SDE) significantly improves the Gen. PPL of LangFlow compared to the deterministic sampler (ODE), matching previous observation to image diffusion models~\citep{karras2022elucidating}.

\begin{table*}[t]
\centering
\caption{Zero-shot perplexities ($\downarrow$) of models trained for 1M steps on OWT. All perplexities for DLMs are upper bounds. $\dagger$ Taken from \citet{sahoo2024simple}. $\P$ Taken from \citet{lou2023discrete}, where these baselines were trained for 1.3M steps, whereas ours and the other baselines were trained for 1M steps. $^{\ddagger}$ Taken from \citet{sahoo2025diffusion}. We use \colorbox{gold!60}{gold}, \colorbox{silver!60}{silver}, and \colorbox{bronze!60}{bronze} to denote the 1st, 2nd, and 3rd place among DLMs, respectively.}
\label{tab:zeroshot_ppl}
\small
\setlength{\tabcolsep}{7.5pt}
\begin{tabular}{lccccccc}
\toprule
\textbf{Model} & \textbf{PTB} & \textbf{Wikitext} & \textbf{LM1B} & \textbf{Lambada} & \textbf{AG News} & \textbf{Pubmed} & \textbf{Arxiv} \\
\midrule
\multicolumn{8}{l}{\textit{Autoregressive}} \\
Transformer$^{\dagger}$ & 82.05 & 25.75 & 51.25 & 51.28 & 52.09 & 49.01 & 41.73 \\
\midrule
\multicolumn{8}{l}{\textit{Diffusion (Absorbing-state)}} \\
SEDD Absorb$^{\dagger}$ & 100.09 & 34.28 & \cellcolor{silver!60}68.20 & 49.86 & \cellcolor{silver!60}62.09 & \cellcolor{bronze!60}44.53 & \cellcolor{bronze!60}38.48 \\
D3PM Absorb$^{\P}$ & 200.82 & 75.16 & 138.92 & 93.47 & -- & -- & -- \\
MDLM$^{\dagger}$ & \cellcolor{bronze!60}95.26 & \cellcolor{silver!60}32.83 & \cellcolor{gold!60}67.01 & \cellcolor{silver!60}47.52 & \cellcolor{gold!60}61.15 & \cellcolor{gold!60}41.89 & \cellcolor{gold!60}37.37 \\
\midrule
\multicolumn{8}{l}{\textit{Diffusion (Uniform-state / Continuous)}} \\
SEDD Uniform$^{\ddagger}$ & 105.51 & 41.10 & 82.62 & 57.29 & 82.64 & 55.89 & 50.86 \\
Plaid$^{\P}$ & 142.60 & 50.86 & 91.12 & 57.28 & -- & -- & -- \\
UDLM$^{\ddagger}$ & 112.82 & 39.42 & 77.59 & 53.57 & 80.96 & 50.98 & 44.08 \\
Duo$^{\ddagger}$ & \cellcolor{silver!60}89.35 & \cellcolor{bronze!60}33.57 & 73.86 & \cellcolor{bronze!60}49.78 & \cellcolor{bronze!60}67.81 & \cellcolor{silver!60}44.48 & 40.39 \\
\midrule
\textbf{LangFlow} & \cellcolor{gold!60}\textbf{81.20} & \cellcolor{gold!60}\textbf{32.28} & \cellcolor{bronze!60}\textbf{68.21} & \cellcolor{gold!60}\textbf{46.93} & \textbf{69.41} & \textbf{46.74} & \cellcolor{silver!60}\textbf{38.47} \\
\bottomrule
\end{tabular}
\vspace{-0.4cm}
\end{table*}

\textbf{Zero-shot transfer.} We evaluate zero-shot transfer following the protocol of prior works~\citep{radford2019language,lou2023discrete,sahoo2024simple,sahoo2025diffusion}, where models trained on OWT are evaluated on a diverse set of downstream corpora, including PTB, Wikitext, LM1B, Lambada, AG News, PubMed, and Arxiv. As shown in Table~\ref{tab:zeroshot_ppl}, LangFlow achieves strong zero-shot performance among DLMs, ranking first on PTB, Wikitext, and Lambada while remaining competitive across all other domains. Although MDLM attains more first-place results overall, LangFlow outperforms the autoregressive Transformer on 4 out of 7 benchmarks and MDLM on 3 out of 7 benchmarks.

\textbf{Verifiable structured generation (Sudoku).}
Beyond language-model-based evaluation metrics, we further evaluate
LangFlow on Sudoku~\citep{kim2025fine,ben2026accelerated,deschenaux2026language},
a verifiable structured generation task that requires global consistency
across the generated sequence.
We report exact-match accuracy on 2,000 validation puzzles per difficulty
level, where a prediction is considered correct only if all solution
tokens match the ground truth.
We compare against MDLM, FLM, and Duo, trained on the same datasets
following the s-FLM~\citep{deschenaux2026language} implementation.
We report two LangFlow variants: auto-Gumbel learns the noise schedule, while fixed-Gumbel fixes it to
$\mathrm{Gumbel}(4.7,1.0)$ for both training and sampling. 
As shown in Table~\ref{tab:sudoku}, LangFlow consistently achieves
stronger accuracy across all difficulty levels.
Notably, even with only 5K training steps (1/4 of the
training budget) and 16 NFEs (1/11 of the sampling budget; see
\cref{app:sudoku} for how NFE is counted),
% LangFlow surpasses all full-budget baselines on medium and hard
% puzzles.
These results provide complementary, fully verifiable evidence
that embedding-space continuous diffusion can effectively
capture global constraints and generate consistent structured outputs.
Additional evaluation details and an analysis of the interaction
between self-conditioning and embedding geometry are provided
in \cref{app:sudoku}.

\begin{table*}[t]
\centering
\caption{
\textbf{Sudoku exact-match accuracy (\%, $\uparrow$).}
Evaluation on 2,000 validation puzzles per difficulty, where all 89
solution tokens must be correct.
Training steps and NFE indicate training and sampling budgets.
auto-Gumbel learns the noise schedule; fixed-Gumbel fixes it to
$\mathrm{Gumbel}(4.7,1.0)$.
}
\label{tab:sudoku}
\small
\setlength{\tabcolsep}{9pt}
\renewcommand{\arraystretch}{1.08}
\begin{tabular}{lccccc}
\toprule
\textbf{Model}
& \textbf{Training Steps}
& \textbf{NFE}
& \textbf{Easy}
& \textbf{Medium}
& \textbf{Hard} \\
\midrule
MDLM & 20K & 180 & 91.45 & 71.75 & 49.70 \\
FLM  & 20K & 180 & 90.60 & 78.70 & 48.45 \\
Duo  & 20K & 180 & 96.15 & 82.40 & 52.10 \\
\midrule
\textbf{LangFlow (auto-Gumbel)}  & 20K & 180 & \textbf{97.50} & 88.45 & 61.05 \\
\textbf{LangFlow (fixed-Gumbel)} & 20K & 180 & 97.30 & \textbf{92.45} & \textbf{76.00} \\
\midrule
\textbf{LangFlow (auto-Gumbel)}  & \textbf{5K} & 16 & 96.00 & 87.25 & 59.90 \\
\bottomrule
\end{tabular}
\vspace{-3mm}
\end{table*}

% \begin{wrapfigure}{l}{0.6\linewidth}
% \centering
% \includegraphics[width=\linewidth]{figs/entropy_zscore_langflow.pdf}
% \caption{Entropy deviation of generated samples reported as $z$-scores relative to the LM1B training distribution. 
% Red dashed lines indicate the $\pm3\sigma$ interval. Plaid, UDLM, and Duo fall beyond $-3\sigma$, indicating abnormally low entropy.}
% \label{fig:entropy}
% \vspace{-0.3cm}
% \end{wrapfigure}

% \textbf{Sample entropy \& Plaid} Sample entropy of AR, discrete diffusion, and Emb-DLM all locate within the $3\sigma$ range around the entropy of the training data. The only exception is Plaid, which corresponds to our MSE objective. Our results show that Plaid achieves surprisingly good Gen.\,PPL of 77.3 and reduces the state-of-the-art PPL of continuous diffusion by $4\times$. This suggests that Plaid has even stronger performance in language modeling than discrete diffusion. However, Section~\ref{sec:4.1} shows that the MSE objective of Plaid leads to mode collapse in the embedding layer. As a result, the sample entropy of Plaid suffers a clear drop. Considering that this entropy drop may cause unpredictable issues in a scaled-up model, we exclude Plaid in our comparison.

\label{sec:5.2}

%% file: sections/conclusion.tex
\section{Conclusion}
\label{sec:conclusion}
We present LangFlow, a principled framework for embedding-space diffusion language modeling grounded in Flow Matching. Our formulation provides a theoretically grounded cross-entropy objective, an ODE-based NLL bound for likelihood evaluation, and a language-specific design space characterized by information-uniform noise scheduling and effective self-conditioning. Empirically, LangFlow achieves strong performance on large-scale benchmarks. It surpasses state-of-the-art discrete diffusion at the same scale in generative quality. These results provide evidence that continuous diffusion can be a competitive paradigm for language modeling. To summarize, LangFlow establishes a stronger baseline for continuous diffusion language modeling, enabling future extensions to build on a foundation comparable to discrete diffusion.

% We present LangFlow, a continuous diffusion language model that operates in the token embedding space. By constructing a novel NLL bound and applying optimal training techniques, LangFlow becomes the first continuous diffusion to match the state-of-the-art discrete diffusion in PPL and Gen. PPL, for the first time indicating the comparable potential of continuous diffusion language models.

\textbf{Limitations \& future work.}
Although LangFlow achieves strong perplexity and generative perplexity, its sample entropy remains lower than that of certain discrete baselines, consistent with observations in other continuous diffusion approaches. In our quantitative and qualitative evaluations (see \cref{app:exp:eval}), we do not observe noticeable degradation in sample quality attributable to this difference. However, reduced entropy may have subtle effects that emerge at larger scales, which fall beyond the scope of this work. We leave further investigation of these effects to future research.

\begin{ack}
Research was supported in part by the Molecule Maker
Lab Institute: An AI Research Institutes program supported by NSF under Award No. 2505932. Any opinions, findings, and conclusions or
recommendations expressed in this publication are those of
the author(s).
\end{ack}
% \textbf{Limitations} Although LangFlow achieves both PPL and Gen.\,PPL which match or outperform those of state-of-the-art discrete diffusion, we find out that its sample entropy is generally lower than other baselines, as shown in Table~\ref{tab:ppl_entropy}. However, we find that increasing the rate of self conditioning from 0.25 to 0.50 during training could increase the sample entropy on LM1B, at a cost of higher PPL. Hence, this issue should be solved by tuning the self-conditioning rate for a sweet point to trade off PPL and entropy, which we leave as future work.  

% \begin{wraptable}{r}{0.55\textwidth}
% \vspace{-10pt}
% \centering
% \caption{PPL and entropy on LM1B and OWT.}
% \label{tab:ppl_entropy}
% \small
% \setlength{\tabcolsep}{3pt}
% \begin{tabular}{lcccc}
% \toprule
% \textbf{Model} & \multicolumn{2}{c}{\textbf{LM1B}} & \multicolumn{2}{c}{\textbf{OWT}} \\
%  & \textbf{PPL}$\downarrow$ & \textbf{Entropy}$\downarrow$ & \textbf{PPL}$\downarrow$ & \textbf{Entropy}$\downarrow$ \\
% \midrule
% MDLM & 31.0 & 4.32 & 23.1 & 5.62 \\
% Duo & 33.6 & 4.31 & 25.2 & 5.55 \\
% \midrule
% LangFlow (SC 0.25) & 30.0 & 4.30 & 24.6 & 5.25 \\
% LangFlow (SC 0.50) & 31.7 & 4.33 & - & - \\
% \bottomrule
% \end{tabular}
% \vspace{-10pt}
% \end{wraptable}

%% file: appendix/related.tex
\section{Additional related work}
\label{app:related-work}
\textbf{Discrete diffusion language models.}
Discrete diffusion operates on categorical data in a discrete space by jumping between discrete states~\citep{austin2021structured,campbell2022continuous,he2023diffusionbert,lou2023discrete,sahoo2024simple,schiff2024simple,ou2024your,sahoo2025diffusion,sahoo2025esoteric,von2025generalized,wang2026generalized,liu2026balancing}. As the mainstream of diffusion language models, discrete diffusion supports scalable language modeling~\citep{gong2024scaling,nie2025large,ye2025dream} with parallel decoding~\citep{cheng2025sdar}. With semi-autoregressive techniques like block diffusion~\citep{arriola2025block}, discrete diffusion morphs into an alternative to AR models with substantially higher parallelism.

\textbf{Continuous diffusion language models.} 
Inspired by the success of continuous diffusion in images and videos, efforts have been made to extend continuous diffusion to language. One approach is to construct diffusion on the simplex of token probabilities~\citep{cheng2024categorical,song2025shortlisting,cheng2025alpha,jo2025continuous}, which struggles to learn the score function at scale with sparse signals due to the curse of dimensionality. Another group of methods diffuse embeddings, and our method falls into this family. 

Most existing embedding-space DLMs proposed their training objectives heuristically, making rigorous evaluation of their negative log-likelihood (NLL) difficult~\citep{li2022diffusion,dieleman2022continuous}, which is the mainstream evaluation metric in language modeling. Plaid~\citep{gulrajani2023likelihood} provides a tractable SDE-based upper bound for NLL, but its objective is cumbersome to optimize, requiring dynamically sliced batches to balance different loss terms. Also, Plaid does not yield comparable performance to discrete counterparts at the scale of OpenWebText. Recent advances in continuous DLMs indicate its potential in few-step generation and achieve sample quality on par with discrete diffusion~\citep{lee2026one,shen2026codar}, but these methods still lack proper ODE-based evaluation bounds for NLL. In contrast to the literature, we connect the cross-entropy objective to Bregman divergence Flow Matching, derive a novel ODE-based upper bound of NLL, and improve the design choices accordingly.

Prior to our information-uniform principle, a heuristic trick for noise scheduling was importance sampling~\citep{nichol2021improved}, which was first introduced to embedding-space DLMs by \citet{gong2022diffuseq}. Importance sampling allocates the training budget according to the loss distribution over different noise levels. By contrast, our noise scheduling principle allocates the training budget according to the distribution of loss derivatives. While our principle is theoretically grounded, importance sampling also allocates more training budget to noisy regions, which roughly resembles our approach empirically. That explains why it yields considerable improvements compared to uniform scheduling.

\textbf{Variational flow matching.} Recent work on Variational Flow Matching (VFM)~\citep{eijkelboom2024variational} introduces a variational view of Flow Matching through auxiliary posteriors. Exponential-Family VFM (EF-VFM)~\citep{guzman2025exponential} further expands such variational objectives to the whole exponential family via Bregman-divergence-based moment matching. In contrast, LangFlow is derived directly from posterior matching under the objective of minimizing the Bregman divergence in the token space. While the variational machinery is therefore not required in our case, we note that a compatible variational interpretation of LangFlow can nevertheless be constructed, which we provide in Appendix~\ref{app:vfm}.

\textbf{Flow map models.}
Many recent concurrent works~\citep{lee2026one,roos2026categorical,potaptchik2026discrete} explore flow map self-distillation for categorical data, demonstrating the few-step generation capabilities of distilled continuous DLMs.
% A common principle underlying these approaches is to parameterize the flow map such that its outputs remain within the discrete probability simplex or the target categorical space, thereby providing a natural grounding for the cross-entropy objective used in flow-map training, as in our work.
% However, these works primarily focus on language modeling for few-step generation.
While they primarily emphasize accelerating generation, our work focuses on improving the training and evaluation of continuous DLMs themselves, closing the gap between continuous DLMs and their discrete counterparts in both likelihood and sample quality.
By establishing a stronger baseline for continuous DLMs, our work provides a more robust foundation for flow maps and other few-step generation techniques.

%% file: appendix/algo.tex
\section{LangFlow}
\subsection{Algorithms}
\label{app:algorithm}

We summarize the complete training and sampling procedures of LangFlow in Algorithm~\ref{alg:training} and Algorithm~\ref{alg:sampling}, respectively.

\begin{figure*}[ht]
\begin{minipage}[t]{\textwidth}
\begin{algorithm}[H]
    \caption{Training}
    \label{alg:training}
    \begin{algorithmic}[1]
        \REPEAT
        \STATE $\vx \sim p_{\rm data}, \vz = (\ve_{x^{(1)}}, \dots, \ve_{x^{(L)}})$
        \STATE $q = \operatorname{clip}(\text{Uniform}(0, 1), 10^{-5}, 1-10^{-5})$ using low-discrepancy sampler
        \STATE $\gamma = \operatorname{stopgrad}(P_\mu - P_\beta\log(-\log q))$
        \STATE $\alpha_\gamma = \sqrt{\text{sigmoid}(-\gamma)}, \sigma_\gamma = \sqrt{\text{sigmoid}(\gamma)}$
        \STATE $\vz_\gamma \sim \mathcal{N}(\alpha_\gamma\vz, \sigma_\gamma^2\mI)$
        \IF{$\text{Bernoulli}(p_{\rm SC})$}
            \STATE $\hat\vx = \hat\vx_{\vtheta}(\vz_\gamma, \gamma)$, self-conditioning: $\bm 0$
            \STATE $\hat{\vz}^{(i)} = \mE\T \hat\vx^{(i)}$ for each $i$
            \STATE Stop gradient on $\hat\vz$
        \ELSE
            \STATE $\hat\vz = \bm 0$
        \ENDIF
        \STATE $\hat\vx = \hat\vx_{\vtheta}(\vz_\gamma, \gamma)$, self-conditioning: $\hat\vz$
        \STATE $\mathcal{L}_{\rm CE} = -\frac{1}{L}\sum_{i=1}^L \log\hat x^{(i, x^{(i)})}$
        \STATE $H_\gamma = H_{+\infty} \cdot \exp(-\exp(-(\gamma-P_\mu)/P_\beta))$
        \STATE $\mathcal{L}_{\rm Scheduler} = (\operatorname{stopgrad}(\mathcal{L}_{\rm CE})-H_\gamma)^2$
        \STATE Take an optimizer step on $\mathcal{L}_{\rm CE}+\mathcal{L}_{\rm Scheduler}$
        \UNTIL{converged}
    \end{algorithmic}
\end{algorithm}
\end{minipage}

\begin{minipage}[t]{\textwidth}
\begin{algorithm}[H]
    \caption{Euler sampling}
    \label{alg:sampling}
    \begin{algorithmic}[1]
        \FOR{$k=0$ \TO $N$}
            \STATE $q=\operatorname{clip}(1-k/N, 10^{-5}, 1-10^{-5})$
            \STATE $\gamma_k=P_\mu - P_\beta\log(-\log q)$
            \STATE $\alpha_k=\sqrt{\operatorname{sigmoid}(-\gamma_k)}$
            \STATE $\sigma_k=\sqrt{\operatorname{sigmoid}(\gamma_k)}$
        \ENDFOR
        \STATE $\vz_0 \sim \mathcal{N}(\bm{0}, \sigma_0^2\mI)$
        \STATE $\hat\vz=\bm 0$
        \FOR{$k=0$ \TO $N-1$}
            \STATE $\hat\vx = \hat\vx_{\vtheta}(\vz_{k}, \gamma_k)$, self-conditioning: $\hat\vz$
            \STATE Update $\hat{\vz}^{(i)} = \mE\T \hat\vx^{(i)}$ for each $i$
            \STATE $\vz_{k+1} = \sigma_{k+1} \left( \frac{\vz_k}{\sigma_k} + \left( \frac{\alpha_{k+1}}{\sigma_{k+1}} - \frac{\alpha_{k}}{\sigma_{k}} \right)\hat{\vz} \right)$
        \ENDFOR
        \STATE $\hat\vx = \hat\vx_{\vtheta}(\vz_N, \gamma_N)$, self-conditioning: $\hat\vz$
        \STATE $x^{(i)} = \arg\max \hat\vx^{(i)}$
        \RETURN $\vx$.
    \end{algorithmic}
\end{algorithm}
\end{minipage}
\end{figure*}

\subsection{Numerical solver}
\label{app:exp:solver}
Generative modeling ODEs frequently rely on the assumption of specific invariants over each discrete integration step. In alignment with our sampling design, we use an Euler solver under the premise that the denoised embedding $\hat\vz_\vtheta(\vz_\gamma, \gamma)$ remains constant within each small interval. To facilitate this discrete integration over appropriate schedule parameters, we perform the following parameterization:
\begin{align}
\frac{\dif\vz_\gamma}{\dif\gamma} = \frac{\dot\sigma_\gamma}{\sigma_\gamma}\vz_\gamma + \frac{\dot\alpha_\gamma\sigma_\gamma-\alpha_\gamma\dot\sigma_\gamma}{\sigma_\gamma}\hat\vz_\vtheta \implies \dif\left(\frac{\vz_\gamma}{\sigma_\gamma}\right) = \hat\vz_\vtheta \dif\left(\frac{\alpha_\gamma}{\sigma_\gamma}\right). \label{eq:solver}
\end{align}

\subsection{Variational view of LangFlow}
\label{app:vfm}

In the main text, LangFlow is derived directly from a Bregman-divergence objective in token space. Nevertheless, LangFlow also admits a variational interpretation closely related to Variational Flow Matching (VFM)~\citep{eijkelboom2024variational}. In this appendix, we make this connection explicit.

Recall that the clean sequence is $\vx=(x^{(1)},\dots,x^{(L)})$, and the corresponding clean embedding sequence is $\vz = (\ve_{x^{(1)}}, \dots, \ve_{x^{(L)}}) \in \mathbb{R}^{L\times D}$,
where $\mE \in \mathbb{R}^{V\times D}$ is the token embedding matrix. Along the $\gamma$-path, the noisy latent state is
\begin{equation}
    \vz_\gamma = \alpha_\gamma \vz + \sigma_\gamma \vepsilon,\qquad \vepsilon \sim \mathcal{N}(\bm 0, \mI).
\end{equation}
This induces a true conditional distribution over clean token sequences given the noisy state, namely the endpoint posterior $p(\vx \mid \vz_\gamma)$.

We define a variational family over token sequences conditioned on $\vz_\gamma$ by
\begin{equation}
q_\vtheta(\vx \mid \vz_\gamma,\gamma)
:= \prod_{i=1}^L q_\vtheta^{(i)}(x^{(i)} \mid \vz_\gamma,\gamma),
\label{eq:var_factorized_q}
\end{equation}
where each factor is parameterized by the model prediction
\begin{equation}
q_\vtheta^{(i)}(x^{(i)}=k \mid \vz_\gamma,\gamma)
:= \hat x_\vtheta^{(i,k)}(\vz_\gamma,\gamma).
\label{eq:var_factorized_q_token}
\end{equation}

Consider the negative conditional log-likelihood of this variational family:
\begin{align}
\Ls_{\rm VFM}(\vtheta)
&:= \mathbb{E}_{\gamma\sim\pi, \vz_\gamma}
\Bigl[
-\frac{1}{L}\log q_\vtheta(\vx \mid \vz_\gamma,\gamma)
\Bigr].
\label{eq:var_langflow_obj_joint}
\end{align}
Using the factorization in \cref{eq:var_factorized_q},
\begin{align}
-\log q_\vtheta(\vx \mid \vz_\gamma,\gamma)
&=
-\sum_{i=1}^L \log q_\vtheta^{(i)}(x^{(i)} \mid \vz_\gamma,\gamma) \\
&=
-\sum_{i=1}^L \log \hat x_\vtheta^{(i,x^{(i)})}(\vz_\gamma,\gamma).
\end{align}
Substituting this into \cref{eq:var_langflow_obj_joint}, we obtain
\begin{equation}
\Ls_{\rm VFM}(\vtheta)
=
\mathbb{E}_{\gamma\sim\pi, \vz_\gamma}
\left[
-\frac{1}{L}\sum_{i=1}^L
\log \hat x_\vtheta^{(i,x^{(i)})}(\vz_\gamma,\gamma)
\right]
=
\Ls_{\rm CE}(\vtheta).
\label{eq:var_equals_ce}
\end{equation}
Therefore, the LangFlow training objective is exactly the negative log-likelihood of a factorized variational approximation to the endpoint posterior.

\input{appendix/bregman.tex}

%% file: appendix/bregman.tex
\subsection{Theoretical connection to Bregman divergence}
\label{app:bregman-divergence}

In \cref{sec:theory.bregman-fm}, we introduced our cross-entropy training objective for Flow Matching on categorical data. Here, we provide a detailed derivation showing that this objective is a special case of \textit{Bregman divergence} minimization~\citep{guzman2025exponential}.

For any convex function $f$, the Bregman divergence is defined as
\begin{equation}
    \mathcal{D}_f(\vp, \vq) = f(\vp) - f(\vq) - \nabla f(\vq) \cdot (\vp - \vq),
\end{equation}
where $\vp$ denotes the target distribution and $\vq$ denotes the predictor.
Using the same notation as in \cref{sec:theory.bregman-fm}, we define a generic training objective along the $\gamma$-path that minimizes the expected Bregman divergence between the one-hot data distribution and the model's prediction:
\begin{equation}
\label{eq:bregman_obj_app}
  \Ls_f(\vtheta) = \mean_{\gamma \sim \pi, \vz_\gamma}\!\left[\frac{1}{L}\sum_{i=1}^L
  \mathcal{D}_f\!\left(\bm1_{x^{(i)}}, \hat\vx_\vtheta^{(i)}(\vz_\gamma,\gamma)\right)\right].
\end{equation}

A key property of Bregman divergence is that when the target $\vp$ is random and the predictor $\vq$ is fixed, we have:
\begin{align}
  \mean[\mathcal{D}_f(\vp, \vq)] &= \mean[f(\vp) - f(\vq) - \nabla f(\vq) \cdot (\vp - \vq)] \\
  &= \mean[f(\vp)] - f(\vq) - \nabla f(\vq) \cdot (\mean[\vp] - \vq) \\
  &= \mean[f(\vp)] - f(\mean[\vp]) + \mathcal{D}_f(\mean[\vp], \vq).
\end{align}
Note that $\mean[\bm1_{x^{(i)}}\mid \vz_\gamma]=P(x^{(i)}\mid \vz_\gamma)$. Applying the above identity to $\vp=\bm1_{x^{(i)}}$ and $\vq=\hat\vx_\vtheta^{(i)}(\vz_\gamma,\gamma)$ yields (writing $\hat\vx_\vtheta^{(i)}(\vz_\gamma,\gamma)$ as $\hat\vx_\vtheta^{(i)}$ for brevity):
\begin{align}
  \mean[\mathcal{D}_f(\bm1_{x^{(i)}}, \hat\vx_\vtheta^{(i)}) \mid \vz_\gamma] &= \mathcal{D}_f(P(x^{(i)}\mid \vz_\gamma), \hat\vx_\vtheta^{(i)}) + \mean[f(\bm1_{x^{(i)}}) \mid \vz_\gamma] - f(P(x^{(i)}\mid \vz_\gamma)) \\
  &= \mathcal{D}_f(P(x^{(i)}\mid \vz_\gamma), \hat\vx_\vtheta^{(i)}) + \text{const}.
\end{align}
Therefore, minimizing $\Ls_f$ is equivalent (up to an additive constant) to posterior matching: for each $\vz_\gamma$, the minimizer of the expected divergence $\mean[\mathcal{D}_f(\bm1_{x^{(i)}}, \vq)\mid \vz_\gamma]$ over $\hat\vx_\vtheta^{(i)}$ is exactly the posterior $P(x^{(i)}\mid \vz_\gamma)$.

The cross-entropy (CE) objective~\citep{dieleman2022continuous,eijkelboom2024variational} used in our LangFlow framework can be recovered as a special case of this generic Bregman divergence objective. By choosing $f(\vp)=\vp\cdot\log\vp$\footnote{We use the convention $\lim_{x\to 0^+} x\log x = 0$ for any elements where $x=0$.}, the Bregman divergence becomes:
\begin{equation}
  \mathcal{D}_f(\vp, \vq) = \vp \cdot \log \vp - \vp \cdot \log \vq.
\end{equation}
For a one-hot target $\vp=\bm1_{x^{(i)}}$, this simplifies to:
\begin{equation}
  \mathcal{D}_f(\bm1_{x^{(i)}},\vq) = -\log q^{(x^{(i)})}.
\end{equation}
Substituting this back into the generic objective \cref{eq:bregman_obj_app}, we obtain:
\begin{equation}
  \Ls_f(\vtheta) = \mean_{\gamma \sim \pi, \vz_\gamma}\!\left[\frac{1}{L}\sum_{i=1}^L -\log \hat x_\vtheta^{(i, x^{(i)})}(\vz_\gamma, \gamma)\right],
\end{equation}
which is exactly the cross-entropy objective $\Ls_{\rm CE}(\vtheta)$ defined in \cref{eq:ce_loss}.
Hence, the CE loss is a principled special case of Bregman-divergence minimization for categorical Flow Matching.

%% file: appendix/proof.tex
\section{Proofs}

% \subsection{\cref{thm:ode_ppl_bound}}
\label{app:proof_ode_ppl}

\odePplBound*

\begin{proof}
First, we define notation for the proof in \cref{tab:notations}. The model first generates the noisiest initial state $\vz_b$ from the prior distribution $p_b(\vz_b; \vtheta) = \N(\vz_b; \bm{0}, \sigma_b^2\mI)$. It then solves the probability flow ODE (\cref{eq:solver}) to obtain the least noisy continuous state $\vz_a$. Finally, it generates the discrete sequence $\vx$ via $P_\vtheta(\vx \mid \vz_a)=\prod_{i=1}^L \hat\vx_\vtheta^{(i, x^{(i)})}(\vz_a, a)$. Overall, the likelihood of the data $\vx$ is $P_\vtheta(\vx) = \int p_a(\vz_a; \vtheta) P_\vtheta(\vx \mid \vz_a) \dif \vz_a$.

\begin{table}[htbp]
\centering
\caption{\textbf{Notations for the proof.}}
\label{tab:notations}
\begin{tabular}{ll}
\toprule
\textbf{Notation} & \textbf{Description} \\
\midrule
$P_\vtheta(\vx)$ & The likelihood of the data $\vx$ under the model $\vtheta$ \\
$P_\vtheta(\vx \mid \vz_a)$ & The model-predicted probability of data $\vx$ given the least noisy embeddings $\vz_a$ \\
$p_\gamma(\vz_\gamma; \vtheta)$ & The probability density of the latent variable $\vz_\gamma$ under the model $\vtheta$ and ODE path \\
$q_a(\vz_a \mid \vx)$ & The training distribution of the least noisy state $\vz_a$ given the clean sequence $\vx$ \\
\bottomrule
\end{tabular}
\end{table}

We start with a VAE-style evidence lower bound (ELBO), viewing the least noisy continuous state $\vz_a$ as a latent variable for the discrete data $\vx$. Its prior distribution $p_a(\vz_a; \vtheta)$ is the marginal distribution at $\gamma=a$ induced by the generative ODE from $\gamma=b$. The forward distribution $q_a(\vz_a \mid \vx)=\N(\vz_a; \alpha_a \vz, \sigma_a^2\mI)$ (where $\vz = \mE\T \vx$ is the clean continuous embedding of $\vx$) serves as the encoder, and $P_\vtheta(\vx \mid \vz_a)$ is the decoder:
\begin{align}
\log P_\vtheta(\vx) &= \log \mean_{q_a(\vz_a \mid \vx)} \left[ \frac{p_a(\vz_a; \vtheta) P_\vtheta(\vx \mid \vz_a)}{q_a(\vz_a \mid \vx)} \right] \\
& \ge \mean_{q_a(\vz_a \mid \vx)} \left[ \log\frac{p_a(\vz_a; \vtheta) P_\vtheta(\vx \mid \vz_a)}{q_a(\vz_a \mid \vx)} \right] \\
& = \mean_{q_a(\vz_a \mid \vx)} \left[ \log p_a(\vz_a; \vtheta) + \log P_\vtheta(\vx \mid \vz_a) \right] + \frac{LD}{2}\log(2\pi\e\sigma_a^2),
\end{align}
where the first equality is an importance sampling identity that holds for any distribution $q_a(\vz_a \mid \vx)$ whose support is compatible with $p_a(\vz_a;\vtheta)$, and we use the entropy of the Gaussian distribution to compute $\mean_{q_a(\vz_a \mid \vx)}[-\log q_a(\vz_a \mid \vx)]$ in the second equality.

To compute $\log p_a(\vz_a; \vtheta)$, we utilize the Instantaneous Change of Variables~\citep{lipman2023flow}:
\begin{equation}
    \frac{\dif}{\dif \gamma} \log p_\gamma(\vz_\gamma; \vtheta) = -\nabla_{\vz_\gamma} \cdot \vv_\vtheta(\vz_\gamma, \gamma).
\end{equation}
Integrating both sides gives:
\begin{equation}
    \log p_a(\vz_a; \vtheta) = \log p_b(\vz_b; \vtheta) + \int_a^b \nabla_{\vz_\gamma} \cdot \vv_\vtheta(\vz_\gamma, \gamma) \dif \gamma.
\end{equation}
Note that $\log p_b(\vz_b; \vtheta) = -LD\log(\sqrt{2\pi}\sigma_b)-\|\vz_b\|^2/2\sigma_b^2$. Therefore,
\begin{align}
    \log P_\vtheta(\vx) & \geq \mean_{q_a(\vz_a \mid \vx)} \left[ \log p_a(\vz_a; \vtheta) + \log P_\vtheta(\vx \mid \vz_a) \right] + \frac{LD}{2}\log(2\pi\e\sigma_a^2) \\
    & = \mean_{q_a(\vz_a \mid \vx)} \left[ -\frac{\|\vz_b\|^2}{2\sigma_b^2} + \int_a^b \nabla_{\vz_\gamma} \cdot \vv_\vtheta(\vz_\gamma, \gamma) \dif\gamma + \log P_\vtheta(\vx \mid \vz_a) \right] + \frac{LD}{2}\log\frac{\e\sigma_a^2}{\sigma_b^2}.
\end{align}

The velocity-denoiser relationship given by \cref{eq:velocity_denoiser} is equivalent to
\begin{align}
\vv_\vtheta(\vz_\gamma, \gamma) &= \frac{\dot\sigma_\gamma}{\sigma_\gamma}\vz_\gamma + \frac{\dot\alpha_\gamma\sigma_\gamma-\alpha_\gamma\dot\sigma_\gamma}{\sigma_\gamma}\hat\vz_\vtheta(\vz_\gamma, \gamma) \\
&= \frac{\dot\sigma_\gamma}{\sigma_\gamma}\vz_\gamma - \frac{\alpha_\gamma}{2}\hat\vz_\vtheta(\vz_\gamma, \gamma),
\end{align}
with the second equality following from the fact that taking the derivative of $\gamma=\log(\sigma_\gamma^2/\alpha_\gamma^2)$ yields $1 = 2\dot\sigma_\gamma/\sigma_\gamma - 2\dot\alpha_\gamma/\alpha_\gamma$, which implies $(\dot\alpha_\gamma\sigma_\gamma-\alpha_\gamma\dot\sigma_\gamma)/\sigma_\gamma = -\alpha_\gamma/2$. Therefore, the divergence term can be computed as:
\begin{align}
\nabla \cdot \vv_\vtheta(\vz_\gamma, \gamma) &= LD\frac{\dot\sigma_\gamma}{\sigma_\gamma} - \frac{\alpha_\gamma}{2}\nabla \cdot \hat\vz_\vtheta(\vz_\gamma, \gamma), \\
\int_a^b \nabla \cdot \vv_\vtheta(\vz_\gamma, \gamma) \dif \gamma &= LD\log\frac{\sigma_b}{\sigma_a} - \int_a^b \frac{\alpha_\gamma}{2} \nabla \cdot \hat\vz_\vtheta(\vz_\gamma, \gamma) \dif \gamma.
\end{align}
Therefore,
\begin{align}
    \log P_\vtheta(\vx) & \geq \mean_{q_a(\vz_a \mid \vx)} \left[ -\frac{\|\vz_b\|^2}{2\sigma_b^2} + \int_a^b \nabla_{\vz_\gamma} \cdot \vv_\vtheta(\vz_\gamma, \gamma) \dif\gamma + \log P_\vtheta(\vx \mid \vz_a) \right] + \frac{LD}{2}\log\frac{\e\sigma_a^2}{\sigma_b^2} \\
    & = \mean_{q_a(\vz_a \mid \vx)} \left[ \frac{LD}{2} - \frac{\|\vz_b\|^2}{2\sigma_b^2} - \int_a^b \frac{\alpha_\gamma}{2} \nabla \cdot \hat\vz_\vtheta(\vz_\gamma, \gamma) \dif \gamma + \log P_\vtheta(\vx \mid \vz_a) \right] \\
    & = \mean_{q_a(\vz_a \mid \vx)} \left[ \frac{LD}{2} - \frac{\|\vz_b\|^2}{2\sigma_b^2} + \sum_{i=1}^L \log \hat\vx_\vtheta^{(i, x^{(i)})}(\vz_a, a) - \int_a^b \frac{\alpha_\gamma}{2} \nabla \cdot \hat\vz_\vtheta(\vz_\gamma, \gamma) \dif \gamma \right],
\end{align}
which concludes the proof.
\end{proof}

%% file: appendix/design.tex
\section{Experimental details}
\label{app:design-choices}

\subsection{Model architecture}
\label{app:exp:model}
We train AR, SEDD, MDLM, and Duo with the Duo codebase~\footnote{\url{https://github.com/s-sahoo/duo}}, UDLM with its codebase~\footnote{\url{https://github.com/kuleshov-group/discrete-diffusion-guidance}}, and Plaid with its codebase~\footnote{\url{https://github.com/igul222/plaid}}. All retrained baselines are trained with their own default setups. The networks in AR, SEDD, MDLM, UDLM, and Duo use the same 130M modified DiT architecture, with 12 attention layers, 12 heads, a hidden dimension of 768, and a 128-dimensional time embedding.

Our model architecture is largely the same as our discrete diffusion baselines, with three minor modifications. First, to incorporate the self-conditioning input $\vz_{\rm SC}$, we update the main input $\vz_\gamma$ before the DiT blocks via $\vz_\gamma \leftarrow \vz_\gamma + W_{\rm in}\vz_\gamma + W_{\rm SC}\vz_{\rm SC}$, where the weight matrices $W_{\rm in}$ and $W_{\rm SC}$ are zero-initialized. Second, we normalize our embeddings on a unit sphere and multiply them by $\sqrt{D}=\sqrt{768}$. This strategy aligns the variance of our data with that of the noise, following the practice in latent diffusion~\citep{rombach2022high}. Third, following Plaid~\citep{gulrajani2023likelihood}, we add a tokenwise bias term $r\log p(\vz_\gamma^{(i)} \mid x^{(i)})$ to the predicted logits $\log\hat\vx^{(i)}(\vz_\gamma^{(i)}, \gamma)$, where
\begin{equation}
  \log p(\vz_\gamma^{(i)} \mid x^{(i)}) = -\frac{\| \vz_\gamma^{(i)}-\alpha_\gamma\ve_{x^{(i)}} \|^2}{2\sigma_\gamma^2} = \frac{\alpha_\gamma}{\sigma_\gamma^2}\ve_{x^{(i)}}\T \vz_\gamma^{(i)} + \text{const}.
\end{equation}
$r$ is ramped up from $0$ to $1$ during the first 5000 iterations and kept fixed until the end of training. We note that these modifications do not substantially change the number of network parameters, which remains around 130M.

\subsection{Training details}
\label{app:exp:training}
Training on LM1B is performed on 4 NVIDIA RTX Pro 6000 Blackwell 96GB GPUs with bfloat16 precision and takes 3 days. Training on OWT is performed on 32 NVIDIA A100 80GB GPUs with bfloat16 precision and takes 10 days. On both datasets, we use the AdamW optimizer with a learning rate of $3\times 10^{-4}$, an EMA decay of 0.9999, and a constant learning rate schedule after a 2,500-step linear warmup.  

\subsection{Plaid baseline}
\label{app:plaid}
\textbf{Model architecture.} As mentioned in Section~\ref{sec:5}, Plaid is the only baseline for which we use a different network architecture during training. In particular, we use its own Transformer-based architecture, which differs from the DiT used by LangFlow and the other baselines. Specifically, Plaid adds the time-condition embedding to the token embeddings and uses embedding dimension 16, whereas LangFlow and the other baselines use 768. This leads to 108M versus 130M parameters, while the transformer depth, number of heads, and hidden dimension remain the same. We follow prior work such as Duo~\citep{sahoo2025diffusion} and SEDD~\citep{lou2023discrete} by using the codebase and architecture provided by Plaid for training, because we find that Plaid's performance degrades when modifying its own codebase and architecture. For example, we tested Plaid with embedding dimension 768, which aligned the network parameters to other baselines. However, this larger embedding dimension degraded the performance of Plaid compared to the original 16-dimensional design. This suggests that Plaid is tuned around its original architecture and that using its own codebase and architecture is necessary for a fair comparison. To better illustrate differences in the architecture and training setup, we list the differences in Table~\ref{tab:plaid_comparison}.

\begin{table}[ht]
\centering
\caption{\textbf{Model and training setup.}}
\label{tab:plaid_comparison}
\small
\setlength{\tabcolsep}{6pt}
\renewcommand{\arraystretch}{1.08}
\begin{tabular}{lcc}
\toprule
\textbf{Feature} & \textbf{Ours \& other baselines} & \textbf{Plaid} \\
\midrule
\# Parameters & 130M & 108M \\
Backbone & Transformer & Transformer \\
Embedding Dim & 768 & 16 \\
Layers / Hidden Dim & 12 / 768 & 12 / 768 \\
Time Conditioning & AdaLN & Addition \\
Training Budget (Total GPU-hours) & $\sim$292 & $\sim$375 \\
Optimizer & AdamW & AdamW \\
Batch Size & 512 & 512 \\
\bottomrule
\end{tabular}
\end{table}

\textbf{Embedding collapse under mean squared error.}
\label{app:ce-vs-mse}
In \cref{sec:theory}, we established the Cross Entropy (CE) loss based on tokenwise probability prediction. However, a common alternative in previous continuous frameworks is to directly regress the denoised embedding $\vz$ using Mean Squared Error (MSE)~\citep{song2019generative,ho2020denoising,lipman2023flow}, defined as:
\begin{equation}\label{eq:mse-loss}
  \Ls_{\rm MSE}(\vtheta) = \int \lambda(\gamma) \mean\left[ \| \vz - \hat\vz_\vtheta(\vz_\gamma, \gamma) \|_2^2 \right] \dif \gamma.
\end{equation}
This training objective is adopted by prior work such as Diffusion-LM~\citep{li2022diffusion} and Plaid~\citep{gulrajani2023likelihood}. However, we empirically observe collapse in the token embedding layer when using this objective. Specifically, we visualize the distribution of the nearest-neighbor distance (NND) for every token embedding in the vocabulary, which describes how token embeddings are scattered in the space. We include four models: autoregressive Transformer (AR), masked diffusion language model (MDLM), LangFlow with the CE loss, and Plaid with the MSE loss as its main objective. We measure NND using spherical distance after normalizing all embeddings to the unit sphere.

\cref{fig:nn_dist} visualizes the distribution of NND for different models. Although Plaid also employs a cross-entropy decoder loss that prevents the embeddings from collapsing into a single point, its nearest-neighbor distances among embeddings are significantly smaller than those of other models, and the average NND reaches 0.058. Moreover, Plaid exhibits an unusual distribution shape compared to other models, with most values concentrated in the lower half of the range. Both observations indicate collapse of the token embedding space.

We explain these observations intuitively. Given one ground-truth token $x^{(i)}$ and the tokenwise prediction $\hat\vx_\vtheta^{(i)}(\vz_\gamma, \gamma)$ (simply denoted by $\hat\vx$), the MSE objective applies a gradient to each embedding $\ve_k$ as follows:
\begin{equation}
  \nabla_{\ve_k} \| \mE\T \hat\vx - \ve_{x^{(i)}} \|^2 = 2 (\hat x_\vtheta^{(i, k)}-\delta_{ik}) (\mE\T \hat\vx - \ve_{x^{(i)}}).
\end{equation}
A gradient step pushes $\ve_k\ (k \ne x^{(i)})$ closer to $\ve_{x^{(i)}}$ and pulls $\ve_{x^{(i)}}$ closer to $\mE\T \hat\vx$, a weighted mean of other embeddings. As a result, in addition to reducing the probability of wrong tokens, the MSE objective also clusters embeddings of different tokens. This is similar to latent space collapse when using the diffusion loss to update the VAE in image diffusion~\citep{leng2025repa}. This collapse may degrade model expressiveness at scale, which may help explain why Plaid ranks highly on smaller datasets like LM1B but performs substantially worse on large-scale zero-shot tasks (see Table~\ref{tab:diffusion_lm_comparison} and Table~\ref{tab:zeroshot_ppl}). To avoid this degradation, we choose CE as our training objective.

\begin{figure}[t]
  \centering
  \includegraphics[width=0.5\linewidth]{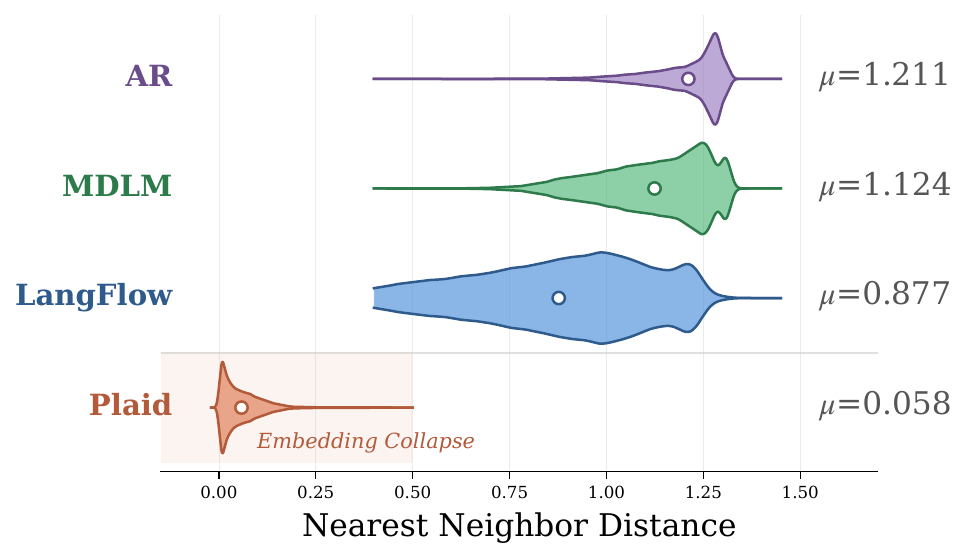}
  \caption{Comparison of nearest neighbor distance (NND) among embeddings of four language modeling architectures: autoregressive (AR), MDLM~\citep{sahoo2024simple}, Plaid~\citep{gulrajani2023likelihood}, and LangFlow. Compared to the other three architectures, the NNDs of Plaid cluster around zero, indicating the indistinguishability of different embeddings (i.e., embedding collapse).}
  \label{fig:nn_dist}
\end{figure}

\subsection{Additional evaluation details}
\label{app:exp:eval}

\textbf{Gen. PPL.} For Gen.\,PPL, we additionally report error bars over multiple evaluation runs.
On LM1B, our model achieves $92.3 \pm 3.8\%$ relative error using two standard deviations.
On OWT, our model achieves $36.5 \pm 1.4\%$ relative error using two standard deviations.

\textbf{Perplexity.} When evaluating the PPL of our model, we use a 128-step Heun-2 solver to compute the log-likelihood integral as well as the initial state $\vz_b$. At each step, we use the Hutchinson's trace estimator to compute the divergence term.

\textbf{Entropy.} Following existing literature~\citep{sahoo2025diffusion}, we evaluate entropy from token frequencies in each sequence, assigning each unique token that appears at least probability $1/128$ for LM1B and $1/1024$ for OWT. We use the same 1024 sequences as in the Gen.\,PPL evaluation for the entropy calculation.

%% file: appendix/exp.tex
\section{Additional experiments}
\label{app:exp}

\subsection{Additional Sudoku Experiments}
\label{app:sudoku}

\paragraph{Data and evaluation protocol.}
We follow the Sudoku benchmark of s-FLM~\citep{deschenaux2026language}.
Puzzles are generated procedurally with a unique solution and grouped into
three difficulty levels by the number of given clues: easy (40), medium
(35), and hard (30). Each difficulty is a separate dataset with 48,000
training and 2,000 validation puzzles (disjoint splits), and every method
trains a separate model per difficulty. A sequence has 180 tokens,
\texttt{[BOS]} puzzle (89) \texttt{[BOS]} solution (89), over a 12-token
vocabulary. The prompt is kept clean during both training and sampling,
and the loss covers the solution tokens only. We report exact-match
accuracy on the 2,000 validation puzzles of each difficulty: a prediction
is correct only if all 89 solution tokens match the ground truth, with no
partial credit.

\paragraph{Baselines and budgets.}
We compare against MDLM~\citep{sahoo2024simple}, FLM~\citep{lee2026one},
and Duo~\citep{sahoo2025diffusion}, re-trained and sampled with the
original s-FLM implementation on the identical datasets rather than
quoting published numbers. All models use the same 8-layer DiT backbone
(width 512), 20K training steps with batch size 256, Adam with learning
rate $3\times10^{-4}$, and EMA $0.9999$; primary comparisons use 180 NFEs.
NFE counts every forward pass of the denoiser. In \cref{alg:sampling},
$N$ Euler steps each evaluate the denoiser once and the final readout
(line~14) adds one more, so NFE $=N+1$: NFE $=180$ and NFE $=16$
correspond to $N=179$ and $N=15$. Self-conditioning reuses the previous
prediction and adds no NFE at sampling time. We additionally evaluate
LangFlow at 5K training steps and 16 NFEs to assess training and sampling
efficiency.

\paragraph{Noise schedules.}
auto-Gumbel is the default LangFlow scheduler: $(P_\mu,P_\beta,H_{+\infty})$
are initialized at $(4.723,\,0.852,\,7.02)$ and learned with
$\mathcal{L}_{\mathrm{Scheduler}}$, giving an initial support of
$\gamma\in[2.64,\,14.53]$. fixed-Gumbel fixes $P_\mu=4.7$ and
$P_\beta=1.0$ for the whole run ($\mathcal{L}_{\mathrm{Scheduler}}$
disabled), with the same $10^{-5}$ quantile clipping, giving
$\gamma\in[2.26,\,16.21]$. Since \cref{alg:sampling} lays out the sampling
grid with the inverse CDF of the same distribution, the two variants
differ in both the training $\gamma$-density and the sampling grid.
Complete results are reported in \cref{tab:sudoku}.

\paragraph{Self-conditioning and embedding geometry.}
To better understand why self-conditioning (SC) is particularly effective
in LangFlow, we hypothesize an \emph{expressivity--readout trade-off}:
learned token embeddings offer greater representational expressivity, but
their non-orthogonal and evolving geometry also makes token readout more
challenging. Self-conditioning may help stabilize this readout during
iterative denoising.

To examine this hypothesis, we conduct a controlled experiment on hard
Sudoku with LangFlow (auto-Gumbel, 20K training steps, 180 NFEs).
Since the vocabulary size is $V=12$ and the embedding dimension is
$D=512$, we can replace the learned token embeddings with a frozen
orthonormal frame, whose maximum pairwise cosine similarity between
distinct embeddings is approximately $3.4\times10^{-8}$. All other factors,
including the dataset, noise scheduler, backbone, training objective, and
training budget, are held fixed.

\begin{table}[t]
\centering
\caption{
\textbf{Effect of embedding geometry on self-conditioning.}
Hard Sudoku exact-match accuracy (\%, $\uparrow$).
$\Delta_{\mathrm{SC}}$ denotes the improvement from enabling self-conditioning.
}
\label{tab:sudoku_sc_geometry}
\small
\setlength{\tabcolsep}{6pt}
\renewcommand{\arraystretch}{1.08}
\begin{tabular}{lccc}
\toprule
\textbf{Embedding Geometry}
& \textbf{w/o SC}
& \textbf{w/ SC}
& $\boldsymbol{\Delta_{\mathrm{SC}}}$ \\
\midrule
Frozen orthonormal & 22.40 & 48.20 & +25.80 \\
Learned embeddings & 17.45 & \textbf{61.05} & \textbf{+43.60} \\
\bottomrule
\end{tabular}
\vspace{-3mm}
\end{table}

As shown in \cref{tab:sudoku_sc_geometry}, self-conditioning substantially improves both embedding geometries, but its gain is considerably larger with learned embeddings (+43.60 versus +25.80 points). Without SC, learned embeddings perform worse than the orthonormal frame (17.45 versus 22.40), suggesting a more challenging readout problem. With SC, learned embeddings instead achieve higher accuracy (61.05 versus 48.20), consistent with their greater representational flexibility becoming beneficial once readout is stabilized.

Moreover, the SC gain with learned embeddings increases with puzzle difficulty: +15.05, +36.65, and +43.60 points on easy, medium, and hard Sudoku, respectively. These observations support the expressivity--readout trade-off within LangFlow, although they do not establish a universal advantage of learned embedding geometry across all continuous diffusion models.

\subsection{Self-conditioning dynamics}
\label{app:self-cond}
Despite the general benefits of self-conditioning, prior works on discrete diffusion language models typically evaluate PPL with self-conditioning disabled. This evaluation protocol has been carried over to continuous diffusion. We empirically find that applying this protocol to continuous diffusion is inappropriate. As shown by the ablation in Table~\ref{tab:sc}, turning self-conditioning on and off for MDLM (discrete) and LangFlow (our continuous model) reveals an asymmetry: for MDLM, self-conditioning improves Gen.\,PPL but degrades PPL, which helps explain its omission in discrete PPL evaluations. For LangFlow, the picture is different: self-conditioning significantly improves both Gen.\,PPL and PPL, which is enough to close the PPL gap between continuous and discrete diffusion in our setting.

\begin{table}[h]
\vspace{-6pt}
\centering
\caption{Self-conditioning ablations on LM1B.}
\label{tab:sc}
\small
\setlength{\tabcolsep}{3pt}
\begin{tabular}{lcccc}
\toprule
\textbf{Model} & \textbf{Gen.\,PPL}$\downarrow$ & $\Delta$ & \textbf{PPL}$\downarrow$ & $\Delta$ \\
\midrule
MDLM & 103.9 & -- & 31.0 & -- \\
\textit{+Self-Conditioning} & 94.9 & -9.0 & 32.7 & +1.7 \\
\midrule
LangFlow & 154.2 & -- & 49.0 & -- \\
\textit{+Self-Conditioning} & 92.3 & -62.0 & 30.0 & -19.0 \\
\bottomrule
\end{tabular}
\vspace{-4pt}
\end{table}

\begin{figure}[htbp]
    \centering
    \includegraphics[width=0.9\linewidth]{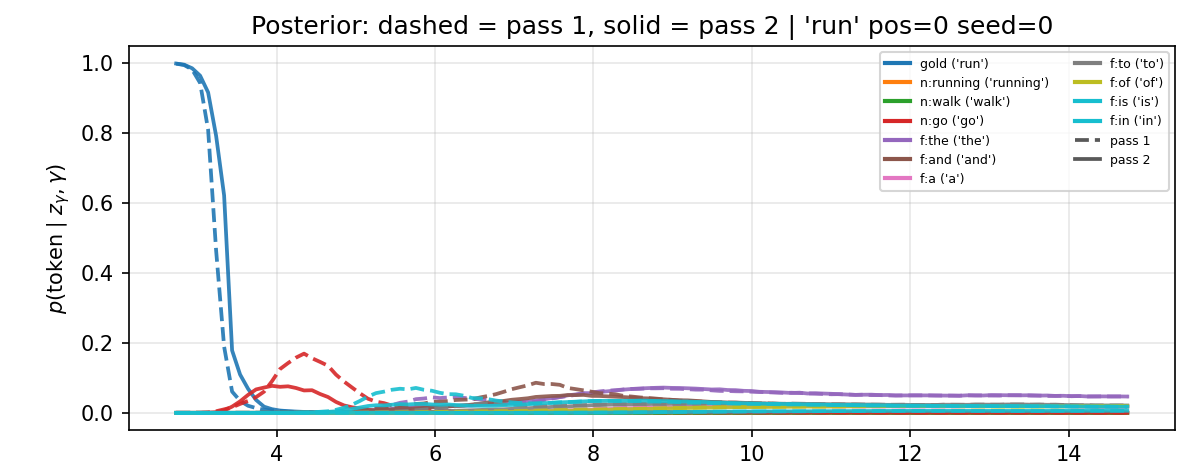}
    \caption{Self-conditioning dynamics on LM1B. The clean token is ``run''. Dashed lines denote pass 1 (without SC), and solid lines denote pass 2 (with SC). We visualize the evolution of $p(x \mid z_\gamma)$ across token groups.} 
    \label{fig:sc_dynamics}
\end{figure}
To better understand the effect of self-conditioning (SC), we analyze how the token posterior $p(x \mid z_\gamma)$ evolves across noise levels $\gamma$. Figure~\ref{fig:sc_dynamics} visualizes this behavior on LM1B for a representative example where the ground-truth token is \textit{``run''}. The dashed curves correspond to the first pass (without SC), while the solid curves correspond to the second pass (with SC).

As $\gamma$ increases, the posterior distribution shifts from the ground-truth token \textit{``run''} to semantically related alternatives such as \textit{``go''}, followed by more frequent syntactic tokens such as \textit{``is''}, \textit{``and''}, and \textit{``the''}. This progression is not random: it reflects a gradual transition from semantic uncertainty to frequency-dominated uncertainty. This suggests that SC on the LM1B checkpoint prevents the model from drifting toward high-frequency but semantically weak tokens during iterative refinement.

\subsection{Quantitative sample quality}
\label{app:gen_ppl}
We additionally report NFE ablations to characterize the trade-off between generation quality and sampling budget. 
All LangFlow results are obtained \emph{without} distillation or any specialized few-step training. We evaluate the same trained checkpoint with different numbers of ODE solver steps, while keeping all other sampling settings fixed.

On LM1B (Table~\ref{tab:nfe_lm1b}), we report only LangFlow, since comparable NFE sweeps from prior continuous and discrete diffusion baselines are not available. For OWT (Table~\ref{tab:owt_compare_nfe}), we compare LangFlow with Duo, MDLM, and SEDD under different NFEs. The LangFlow results are our own evaluations, while the results for Duo, MDLM, and SEDD are quoted from the corresponding table in \citet{sahoo2025diffusion} and included here for reference. The sample entropies of autoregressive baselines on LM1B and OWT are 4.32 and 5.58, respectively.
\begin{table}[h]
\centering
\caption{\textbf{NFE comparison on LM1B.} 
We report the Gen.\,PPL and entropy for LangFlow under different NFEs.}
\label{tab:nfe_lm1b}
\small
\setlength{\tabcolsep}{7pt}
\begin{tabular}{lcc}
\toprule
\textbf{NFE} & \textbf{Gen.~PPL} & \textbf{Entropy} \\
\midrule
128 & 92.25  & 4.31 \\
64  & 104.83 & 4.32 \\
32  & 127.32 & 4.33 \\
16  & 179.60 & 4.35 \\
\midrule
\midrule
\multicolumn{3}{c}{Data Entropy=4.32}\\
\bottomrule
\end{tabular}
\end{table}

\begin{table}[h]
\centering
\caption{\textbf{NFE comparison on OWT.} 
We compare LangFlow with Duo, MDLM, and SEDD under different sampling budgets. The data entropy (5.44) is quoted from \citet{lee2026one}.}
\label{tab:owt_compare_nfe}
\footnotesize
\setlength{\tabcolsep}{2pt}
\renewcommand{\arraystretch}{1.08}
\begin{tabular}{c cc cc cc cc cc}
\toprule
& \multicolumn{2}{c}{\textbf{LangFlow}}
& \multicolumn{2}{c}{\textbf{Duo}}
& \multicolumn{2}{c}{\textbf{SEDD Uniform}}
& \multicolumn{2}{c}{\textbf{MDLM}}
& \multicolumn{2}{c}{\textbf{SEDD Absorb}} \\
\cmidrule(lr){2-3}
\cmidrule(lr){4-5}
\cmidrule(lr){6-7}
\cmidrule(lr){8-9}
\cmidrule(lr){10-11}
\textbf{NFE}
& \textbf{Gen.\,PPL} & \textbf{Entropy}
& \textbf{Gen.\,PPL} & \textbf{Entropy}
& \textbf{Gen.\,PPL} & \textbf{Entropy}
& \textbf{Gen.\,PPL} & \textbf{Entropy}
& \textbf{Gen.\,PPL} & \textbf{Entropy} \\
\midrule
1024 & 36.53 & 5.25 & 77.69  & 5.55 & 99.90  & 5.56 & 104.85 & 5.63 & 105.03 & 5.62 \\
512  & 41.52 & 5.31 & 78.14  & 5.55 & 100.44 & 5.56 & 104.43 & 5.63 & 104.45 & 5.62 \\
256  & 49.24 & 5.37 & 78.62  & 5.55 & 103.41 & 5.56 & 112.70 & 5.66 & 109.82 & 5.63 \\
128  & 60.09 & 5.43 & 80.02  & 5.55 & 105.82 & 5.57 & 120.77 & 5.67 & 117.28 & 5.65 \\
64   & 80.34 & 5.51 & 85.62  & 5.57 & 113.02 & 5.57 & 143.88 & 5.70 & 138.42 & 5.67 \\
% 32   & \textemdash & \textemdash & 96.19  & 5.57 & 125.21 & 5.57 & 196.79 & 5.75 & 184.71 & 5.72 \\
% 16   & \textemdash & \textemdash & 122.78 & 5.58 & 165.66 & 5.58 & 343.33 & 5.81 & 316.33 & 5.77 \\
% 8    & \textemdash & \textemdash & 198.27 & 5.57 & 276.89 & 5.59 & 830.82 & 5.91 & 748.37 & 5.85 \\
\midrule
\midrule
\multicolumn{11}{c}{Data Entropy=5.44}\\
\bottomrule
\end{tabular}
\end{table}
\normalsize
From Table~\ref{tab:owt_compare_nfe}, the sample entropy of LangFlow is lower than that of the baselines under the same number of sampling steps on OWT. This typically suggests that generated text may suffer from repetition that harms semantic quality. However, upon closer inspection, we find that the reduced entropy is primarily due to the high frequency of certain content words within the samples rather than undesirable local repetition.

Specifically, we count the maximum occurrence of a single content word in each sample from \cref{app:samples} and summarize the results in Table~\ref{tab:content_word_freq}. One interesting observation is that the ranking of maximum frequency aligns well with the ranking of sample entropy, and the selected sample from LangFlow contains a content word repeated up to 11 times. 

\begin{table}[htbp]
\centering
\caption{\textbf{Maximum frequency of a single content word in generated samples.} 
We report the most frequent content word and its occurrence count for each method.}
\label{tab:content_word_freq}
\footnotesize
\setlength{\tabcolsep}{7pt}
\renewcommand{\arraystretch}{1.08}
\begin{tabular}{c c c c c}
\toprule
& \textbf{LangFlow} & \textbf{AR} & \textbf{Duo} & \textbf{MDLM} \\
\midrule
\textbf{Word} 
& health & body & NYCFC & Mitchell  \\
\textbf{Count} 
& 11 & 4 & 8 & 13 \\
\textbf{Entropy} 
& 5.3375 & 5.5742 & 5.4827 & 5.2111 \\
\bottomrule
\end{tabular}
\end{table}

Notably, these repeated words are distributed across different sentences and are often far apart, thus having a limited impact on overall semantic quality. We observe that this phenomenon appears in most of LangFlow's samples and even in other continuous DLMs. 
%This suggests that such global repetition of content words may be an inherent characteristic of continuous DLMs, which we leave for future investigation.
This suggests that the entropy gap may be driven by global frequency bias---over-representation of frequent tokens, rather than local degeneration.

These observations indicate that entropy should be interpreted with caution as a standalone metric for generation quality, as it conflates distributional calibration with sequence-level coherence. We leave a deeper investigation of this discrepancy to future work.

\paragraph{Gen.\,PPL--entropy frontier.}
Because Gen.\,PPL can be lowered at the expense of diversity, comparing models at a single sampling configuration conflates sample quality with sample entropy. Following CANDI~\citep{pynadath2025candi}, we therefore scale the output logits of each model with a temperature $\tau$ and sweep $\tau$ to trace the Gen.\,PPL--entropy frontier of LangFlow, MDLM, and Duo on OWT. We use 256 and 1024 NFEs and evaluate 1024 samples for each temperature. \cref{fig:pareto_owt} shows the resulting frontiers.

\begin{figure}[htbp]
    \centering
    \includegraphics[width=\linewidth]{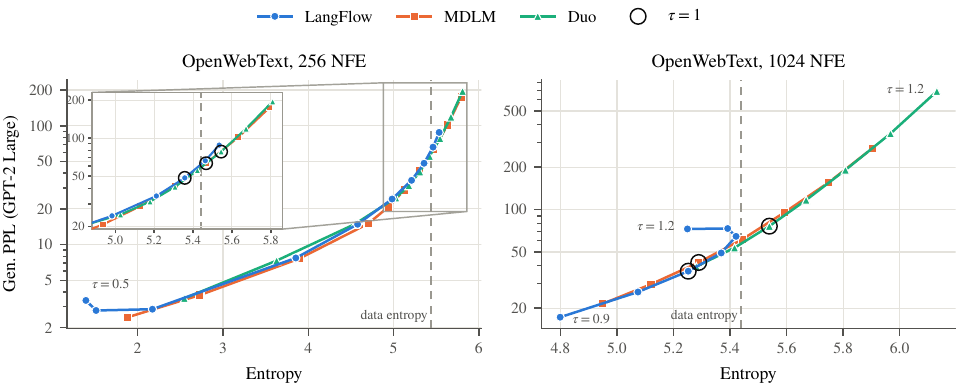}
    \caption{\textbf{Gen.\,PPL--entropy frontier on OWT} under 256 (left) and 1024 (right) NFEs. Each marker is one sampling temperature evaluated on 1024 samples, with $\tau$ ranging from 0.5 to 1.1 for 256 NFEs and from 0.9 to 1.2 for 1024 NFEs. Black circles mark $\tau=1$. The inset in the left panel zooms into the high-entropy regime, and the dashed line denotes the data entropy (5.44) quoted from \citet{lee2026one}. Lower Gen.\,PPL at higher entropy is better.}
    \label{fig:pareto_owt}
\end{figure}

LangFlow achieves a frontier comparable to those of MDLM and Duo under both sampling budgets. For example, under 256 NFEs, LangFlow reaches a Gen.\,PPL of 66.17 at an entropy of 5.46 with $\tau=1.05$, on par with MDLM, which reaches 63.21 at an entropy of 5.47 with $\tau=1.0$. Under 1024 NFEs, the three frontiers nearly overlap within the entropy range they share. However, LangFlow requires a higher temperature than the discrete baselines to reach a similar Gen.\,PPL and entropy, consistent with the lower entropy of its untempered samples discussed above. Moreover, under 1024 NFEs, the entropy of LangFlow saturates at around 5.42 with $\tau=1.1$; further increasing $\tau$ lowers the entropy instead of raising it, so the frontier bends backward. At the other end, all methods produce degenerate, low-entropy samples (entropy below 3.7) for $\tau \le 0.6$ under 256 NFEs, where Gen.\,PPL is no longer informative.

%% file: appendix/last.tex
\subsection{Qualitative samples}
\label{app:samples}

\subsubsection{LM1B}
We present qualitative generation samples from different language modeling paradigms.
All samples are generated with a fixed sequence length of 128 tokens. 

\begin{modelbox}{LangFlow}
% 1M-sc
\textbf{Sample 1}
Generative Perplexity: 63.4070; Sample Entropy: 4.3104
\begin{samplebox}
[CLS]lva, a chief negotiator, had smuggled kazakhstan in mid - july. [CLS] live dancing property stars - about half a million people in a rural town in the south - west have claimed hundreds of thousands of pounds for a " beautiful adventure " after a dispute with police over a ban to service. [CLS] two iranian human rights activists have been arrested in china following a violent crackdown on pro - business protesters, a translator reported today. [CLS] disney's \$ 40. 6 million in gross revenue for the three months that ended dec. [CLS] " heady and restrain " also contains a probaquiial reporting that touches [CLS]
\end{samplebox}

\textbf{Sample 2}
Generative Perplexity: 110.6519; Sample Entropy: 4.3081
\begin{samplebox}
[CLS] the science universe, where more than 1, 000 scientists around the world are searching for planets, the observations offer a better future than the old particle accelerator, the lhc - - less than three months after theories ended. [CLS] and while zakib will go round the batsmen down the hook and ask them exactly what price the umpires are unlikely to pay should they all lose, it will be quite free to talk about how the axe will go on the service. [CLS] the us drug company already sells the best - selling single cabay in canada and this week, pfizer has agreed to rebuild its services business, just [CLS]
\end{samplebox}

\textbf{Sample 3}
Generative Perplexity: 67.0254; Sample Entropy: 4.3124
\begin{samplebox}
[CLS] was convicted of facilitating the smuggling of weapons across south africa and mozambique. [CLS] for the go - ahead run in the fifth, dan uggla escaped a one - out, pinch - hitter and fluffed the crowd from the offensive line. [CLS] " i was very disappointed for kred, but i'm glad that mccallen was involved because he gave us money, " he said. [CLS] " this is one of the most challenging recent years in terms of providing excellent quality of service to our customers at a time with sharply peak demand, " said timothy bowling, nicotine usa chief executive officer. [CLS] sergio nouvel [CLS]
\end{samplebox}

\end{modelbox}
\newpage
\begin{modelbox}{AR}
\textbf{Sample 1}
Generative Perplexity: 71.8890; Sample Entropy: 4.3170
\begin{samplebox}%%%%%%%%%%%%%%%%%%%%%%
[CLS] to move the global economy forward, not by selling us debt, " obama said. [CLS] i only ever looked forward to " second life " surfacing to digress from my list. [CLS] on june 20, 1999, vanessa o'neal - - one of the first names us film stars and fashion models, as well as fraternal twins - - leaned on the allison hightowers. [CLS] a 2005 analysis by a united nations agency concluded that the statistics could not be accounted for, given rocket strikes, gunfights, suicide bombings and other taliban atrocities. [CLS] to help maintain her lead in delegates, texas gop chairman [CLS]
\end{samplebox}
\textbf{Sample 2}
Generative Perplexity: 69.8902; Sample Entropy: 4.3869
\begin{samplebox}%%%%%%%%%%%%%%%%%%%%%%
[CLS] s just dream that no judge might someday recognize your trustworthiness, " said pawlowski. [CLS] the father covers basic fieldwork for the company, and fischer is willing to work with the employees. [CLS] after viewing ms. pelosi's statement, aides looked at the public lives of non - vermonters who voted for stevens in a june 6 broadcast in california on the house floor. [CLS] saakashvili's efforts to win over the russian public are boosted by his handling of the war with russia, which appears closer than polls suggest, and by expectations of an improved european union ties. [CLS] that is [CLS]
\end{samplebox}
\textbf{Sample 3}
Generative Perplexity: 91.7882; Sample Entropy: 4.3387
\begin{samplebox}%%%%%%%%%%%%%%%%%%%%%%
[CLS] 000 slowly grew with some drought, while another million kept their fish, one christmas, away. [CLS] three of the men parked a white compact car outside a atleeken apartment complex von maur store, ransacked a women's store and arrested 29 people, including 23 men, cnn reported tuesday. [CLS] u. s. life expectancy - - generally rising for a longer period, maybe reaching 59, and for those already age 70 or older - - has been declining for decades. [CLS] you may not be able to provide a full postal account in the uk, but no matter how exciting your job or job, a [CLS]
\end{samplebox}
\end{modelbox}
% \newpage
\begin{modelbox}{Duo}
\textbf{Sample 1}
Generative Perplexity: 133.2676; Sample Entropy: 4.3306
\begin{samplebox}%%%%%%%%%%%%%%%%%%%%%%
[CLS] growing online television business. [CLS] and we passed it upon ourselves, " said hey backers, an associate professor in the life chemistry business from the university of stony brook. [CLS] the scientists published similar - looking, transmitted formative images and found several different components with new colors and completely zero degree radiation. [CLS] shock and awe is as much as the prokofiev affair. [CLS] when the house finally approved the budget on this week, it came on a fine vote. [CLS] but it was luck that took draney his first professional award. [CLS] mcclellan has a number two, then the proceeds for personal use - - almost [CLS]
\end{samplebox}
\textbf{Sample 2}
Generative Perplexity: 60.0276; Sample Entropy: 4.2995
\begin{samplebox}%%%%%%%%%%%%%%%%%%%%%%
[CLS] accidents are reported in india, including three in passenger vehicles. [CLS] " we had hoped to send it to us straight - a few miles - but it did. [CLS] the majority of books contain negatives on the subject including pornography. [CLS] mayor romero said the group approached his german msp and demanded a " dangerous " act. [CLS] dili abata, indonesia - - leu suhartian, the son of president susilo bambang yudhoyono, yesterday faced opponents for the first time in his 10 years in office. [CLS] yes, more likely to create your anger was the belief that the pastime had become [CLS]
\end{samplebox}
\textbf{Sample 3}
Generative Perplexity: 119.1866; Sample Entropy: 4.4145
\begin{samplebox}%%%%%%%%%%%%%%%%%%%%%%
[CLS] in the us and is it easy to get them through the internet. [CLS] the nhs has suffered several changes after pablo barren was shot to death at morrison\u02bcs publican in serwoo da on november, 1997. [CLS] kudos - - although " grey's anatomy " may be as bad a money - making venture as that : the essential nervous breakdown : at the time of the source coronation for hbo ( no longer ). [CLS] in the 1990s the ranks were swelled with more than 4, 000 people needing kidney transplants only in hospitals and not seen by medics or local clinicians each year. [CLS] the convention'[CLS]
\end{samplebox}
\end{modelbox}
\newpage
\begin{modelbox}{MDLM}
\textbf{Sample 1}
Generative Perplexity: 159.5354; Sample Entropy: 4.4154
\begin{samplebox}%%%%%%%%%%%%%%%%%%%%%%
[CLS], the vote was presumed to be crucial in order to stop others from criticising them. [CLS] even though some have quit, his foreign ministers accepted as he wants to ensure the labour government does not become embroiled in another scandal involving the leak of accusations by a foreign administration against the war. [CLS] chicago ( ap ) - and who kills the hepatitis? [CLS] you confidently cooled your mistress must have terrified them to death only people who killed them in matters of love. [CLS] yet new research hass that one way toward making records use is to completelyxplug usage data from other programs. [CLS] the environmental protection agency held [CLS]
\end{samplebox}
\textbf{Sample 2}
Generative Perplexity: 156.5223; Sample Entropy: 4.4382
\begin{samplebox}%%%%%%%%%%%%%%%%%%%%%%
[CLS] aren't so silly. [CLS] after the long - term flights gone since strong sellers have double - dipping into the legendary following giants hollywood mogul, the group is signaling another change in direction. [CLS] their spirits are spoken strongly about the fate of the republic, although most insist that they call themselves all part ownership. [CLS] wall street pared huge gains tuesday a report said it was safe and banks's tanked as a record 5 - barrel drop in oil prices undermined hopes for crude. [CLS] now this space station is due for offing, and may only have completed the first mission to orbit. [CLS] the teams also played [CLS]
\end{samplebox}
\textbf{Sample 3}
Generative Perplexity: 109.2186; Sample Entropy: 4.3417
\begin{samplebox}%%%%%%%%%%%%%%%%%%%%%%
[CLS] of tight security outside the manmohan hotel. [CLS] the fantastic adrian lulli and joel movies will show for the broadcaster. [CLS] the sleep disorder was beginning healing in men, and said that how they fared among his patients was unclear. [CLS] but barca embark upon a pecking order at the end of the summer. [CLS] in a dramatic 75 - minute battle chelsea took the aggregate lead after just 12 minutes as anelka snatched it clear through frankon - mikel. [CLS] the spokesman said somerset officers were reviewing experts to assess whether there were any potential cause with the rain and wind. [CLS] she is the husband of prince [CLS]
\end{samplebox}
\end{modelbox}
% \newpage
\begin{modelbox}{UDLM}
\textbf{Sample 1}
Generative Perplexity: 181.8081; Sample Entropy: 4.3870
\begin{samplebox}%%%%%%%%%%%%%%%%%%%%%%
[CLS] all political " work " projects he has with the gm group, which distributes construction permits and serves as its storm development adviser for all of northern ireland. [CLS] i am 15 sussex but a straight face has been reported only as an exotic specimen ( we were born in 1936, before the establishment of a rancher, although i have only lived in the one apartment building for 15 years and for about the time storey's shares fell, the " quick shot homemaker " fled town ). [CLS] the part of the swai piana trophy depicts straightaway male hilouds, clustered around a teetering venus sculpture from the late [CLS]
\end{samplebox}
\textbf{Sample 2}
Generative Perplexity: 89.3953; Sample Entropy: 4.2908
\begin{samplebox}%%%%%%%%%%%%%%%%%%%%%%
[CLS] copies of his family and friends online, signed him recommendations at an oxfordshire hotel before he by their secretary. [CLS] it is now close to achieving 25 per cent. [CLS] there must be accountability, and there is accountability, and for israel to use our horns and gas as racial epithets at black males and allow them to wage war or be brought down by fatalities - - you don't understand is a latin american junk. [CLS] 21 of the oldest u. s. military veteran in afghanistan, hasen said. [CLS] peter l. nickles - calif., majority leader of the national committee, said his organization has given the [CLS]
\end{samplebox}
\textbf{Sample 3}
Generative Perplexity: 85.1183; Sample Entropy: 4.2261
\begin{samplebox}%%%%%%%%%%%%%%%%%%%%%%
[CLS] not as impressive after he took the first four shots of the morning. [CLS] chicago ( ap ) - - dan andino posted his third black debut goal of the season early on with goals, helping the oilers tie chicago 2 - 1 saturday night. [CLS] but sadc's achilles heel has been staggering. [CLS] ( ap ) - mike harper sank his 37th foul of the 12 : 65 to cut the hawks lead 10 early in the third over philadelphia. [CLS] " with the fanism we've got two wins on the night in the west and if we don't play against the blues on sunday. [CLS] she will turn point [CLS]
\end{samplebox}
\end{modelbox}
\newpage
\begin{modelbox}{SEDD}
\textbf{Sample 1}
Generative Perplexity: 97.1503; Sample Entropy: 4.2135
\begin{samplebox}%%%%%%%%%%%%%%%%%%%%%%
[CLS] time the nato - u. s. naval coalition took charge to assist admirals operations in the last 24 months, and most of it because of the unexpectedly big jump. [CLS] the la times and alche judges claimed the top 10 nominees, while hudson, 31, and gara, 26 earned the rest. [CLS] fu fologists and readers of rupert murdoch's blog have apparently penned msft in question. [CLS] syracuse seemed certain they would do it against cincinnati, but by their words it didn't matter anymore. [CLS] the supervising officer, raymond gadd, thanked the team of officers who took part and set up the perimeter [CLS]
\end{samplebox}
\textbf{Sample 2}
Generative Perplexity: 56.9544; Sample Entropy: 4.2509
\begin{samplebox}%%%%%%%%%%%%%%%%%%%%%%
[CLS]bi does not provide service to his payroll, nor do he personally do payment to any property. [CLS] a phone message left for ford doctors'spokesperson was not immediately returned. [CLS] " can remember a time when they give you space back, " he said. [CLS] he was also convicted of raping a boy during an 11 - year period in 1975. [CLS] a sampling of the polling numbers shows that about half the population of republicans - - 49 percent to 44 percent, according to an associated press - - expected republicans to tell cnn they were taking more blame than democrats, republicans, and european socialists. [CLS] " absolutely zero, america'[CLS]
\end{samplebox}
\textbf{Sample 3}
Generative Perplexity: 90.7171; Sample Entropy: 4.2873
\begin{samplebox}%%%%%%%%%%%%%%%%%%%%%%
[CLS]. [CLS] this should shore up buyers, even though this month's purchasing managers index reported by the institute of supply association shows a business rate - - a standard gauge of interest in goods - - up 37 points to 86. 5. [CLS] the greatest challenge is investing all the money and the science. [CLS] us treasury will receive the 340 million dollars this time in the form of " cash gains " under the decree of monday after there was no announcement publicly thereafter, the federal reserve inspector general said in a statement. [CLS] last november a - list fund manager emma rowe rushed her three - year - old children to hospital with zero symptoms. [CLS]
\end{samplebox}
\end{modelbox}
% \newpage
\begin{modelbox}{Plaid}
\textbf{Sample 1}
Generative Perplexity: 52.5642; Sample Entropy: 4.2534
\begin{samplebox}%%%%%%%%%%%%%%%%%%%%%%
[CLS] second - quarter growth based on a 3. 5 percent contraction during the second quarter. [CLS] as a mother of three children, involved by her church, urilh said she was a member of the real - life family. [CLS] it's even better when you're at center, or center. [CLS] " our aim is to provide free contraception for the elderly, gay and bisexual and to provide it only to those opposed to treatment. [CLS] profits at british gas and electric, the uk's biggest renewable energy supplier, have been hit by soaring utility costs. [CLS] while a number of the soldiers who died were wounded, [CLS]
\end{samplebox}
\textbf{Sample 2}\\
Generative Perplexity: 70.5757; Sample Entropy: 4.2983
\begin{samplebox}%%%%%%%%%%%%%%%%%%%%%%
[CLS] for community service. [CLS] a march on fleet street to help raise money for the robin croft war memorial in warwickshire has been offered for fundraising. [CLS] he has got britain building an honest society that will thrive on whether locally recognised or managed, and thrive on tyranny, do everything everyone wants us to do. [CLS] the federal reserve allowed its mortgage purchases to come cheap in the form of, well, unsold securities. [CLS] she then handed over her qualification to work in the london building industry, managing to buy a mansion near cockermouth in the early 1950s where she had worked as a office worker. [CLS] dr. david [CLS]
\end{samplebox}
\textbf{Sample 3}
Generative Perplexity: 108.8908; Sample Entropy: 4.3136
\begin{samplebox}%%%%%%%%%%%%%%%%%%%%%%
[CLS] separate locks, even though a national website raised the issue by claiming that none spoke to tuvalu - gobble in the days leading up. [CLS] cuba's victors have long sought to root out u. s. interests from britain. [CLS] heartland asset management ( www. globecorp. com ) was formed for a variety of uses by baker street partners, l. md., a professional information provider ( ifa ) specializing in identifying companies and rationalizing strategies for use in investment management. [CLS] it now weighs about 80ft ( that was given the entire boat's length ), and it can see up to any [CLS]
\end{samplebox}
\end{modelbox}

\newpage
\subsubsection{OWT}
We present qualitative generation samples from different language modeling paradigms.
All samples are generated with a fixed sequence length of 1024 tokens.
\begin{modelbox}{LangFlow}
% 1M-sc
Sample Generative Perplexity: 55.7794; Sample Entropy: 5.3375
\begin{samplebox}
24 of them firmly said OK. Yet the reason is likely to show that many adults worry about the lack of anxiety. Instead, Jeffrey Cohen, a spokesman for the American Association for the Study of Mental Health, and at traveling alone, said membership has received 16 such complaints a week through October. And that is a 6.4 percent increase from the 212 they received in 2000. We've probably heard a lot of discussions about privilege and how much distress it causes,' said Wilson, who worked President Bush on state from 1991 to 2003 and is until recently a director of the American Center for the Study of Mental Health. 'It's more of a taboo discussion.' Not all Americans are concerned about mental health. The U.S. Census has estimated that it receives more reports of mental health than an individual's general household, with more than 175,000 adults. And as employment has grown, mental health has improved in recent decades, the association found. The stereotype 'ginds the link between being more of an outger and an outger,' the association said in a 2011 report published in the Spring issue of Social Psychology. But the health university added, moreger often 'same as a chicken doger, a wharf, a hospital isger, or the shark.' Americans are also about five times more likely to feel physically absent at work, by sleep, matter school or school. About 3 in 10, on average, are far more likely to feel physically absent at work, when working, or eating, at home. Advertisement Continue reading the main story Many of these physical and health problems have an additional effect to a sense of dampiness or loneliness, as well as feelings of insecurity in their well-being or relationships. Overall, mental health is responsible for 35 percent of how much attention a person was put up for in a class course and to at work. Photo Some Americans worry that anxiety is caused by parenting factors. For example, children often walk second to developmental illnesses, a condition that academics like psychologist Holiston Vance of Harvard long said could be linked to suffering more ill expectations. 'The association is worried that (mental health) is at a bit of a toll,' she said in a telephone interview. 'It kills me, you attack myself, you take risks and get frustrated a couple of times.' If science knows where mental distress comes from or for some reason, it takes on it, too. For most of its literature, scientists continue to promote the nature of mental health and how it relates repeatedly to individuals or men in their own countries. In 1929, mental health increased to 25 in Americans, between 1835 and eighteen, according to reports by the G.K.S.D.P.M. of Switzerland. 'Depression has, enough of a sudden, become an American illness,' the association of Americans in 1996 wrote. 'Most Americans seem to believe it to be such a very small effect on the individual.' by An unfinished tunnel in Wetewallie in British Columbia, according to the Coast Watch ferry led by Weston Aaval. Via below: About 15,000 litres, more than two 404-lb worth of concrete, have been cut by an old \$8-million tunnel currently running from Bist Lake to Chamewallie, the Canadian Coast Watch ferry announced Tuesday. Aaval did not yet say for exact state, but it is likely at the cost of the tunnel's upkeep. According to the Herald's Magazine Island and Coast Watch, in 2014, the independent water estimate stated about 8,000 litres of water from the tunnel had been cut in by builders in the process because of extensive flooding. The NTL Drain writes: The ferry said it is still the public cost estimate of the projected tunnel damage and the tunnel will cost between \$200 and \$50,000. Coast Watch captain Etyle Gingler said in a statement the concrete would be pisted toward the west side of the Baltimore Corridor from 20 metres above. The cost of the abandoned tunnel would be \$300, he said. 'It is a pretty significant contribution for anyone who puts it off and exits Prince Edward Pier.' Insurance may cost is a fairly small estimate. Aaval and Coast Watch have yet to predict what the tunnel cost would be, but a think study has shown the tunnel will cost Nova Scotia \$148 billion, if maybe not more than \$1 billion. Moultler, an ale in experimental production, is a flagship brewery in Capweau, Ontario. Previously and initially made and operated by Mabeller Brewing Company.[2] it is known for its personal unique brew, a red-American sour ale. Its blonde cap contains melons,
\end{samplebox}
\end{modelbox}
\newpage
\begin{modelbox}{AR}
% 1M-sc
Sample Generative Perplexity: 45.4673; Sample Entropy: 5.5742
\begin{samplebox}
from occupied DAESH-occupied territories etc.he also emphasized that he not only acted alone and without authorization as he had expected, but also that he knew exactly what he was doing and intended to do. Wabiullah reported that he was grateful to the BRICS and World Bank for their support of Group City. He said that he had yet to receive any official comment from the governments that have been helping to build TBTF institutions in the Arab part of the GCC countries ruled by the NDFFR. He was still troubled when RT reported on the statement made by a Saudi official about his group planting vests in the Saudi installations in Benghazi and other cities along the Aqsa Strait; that Saudi Arabia would pull U.S. troops out of the Qandil on the black oil line to allow classified weapons to cross it. Carl Popper once called his nose deep in a box of food, softly and with a laugh, he would inhale until he was breathing hard. Virtually the entire feast revolved around the sound of the actual sounds of milk creating a soft, soft ooze around the otherwise extraordinarily pale mixture. At any second, it could slip a finger in and rip into bone, dragging and biting and pushing it down through the flesh and causing with every click and twinkling a stray tear squishing itself into a soft, beige greasy match across the texture. Patron can eat whatever they like, body or whole food, mouth or not, body or whole food, all rad. Ninety-nine percent of the creatures in our cave do not feel this way. In fact, many times I have found that people who have a sneaking dislike to unusual games actually like them. This is due, of course, to many of the aforementioned viruses impermanent at certain phenotypes, and of course because virtual reality cannot be manipulated without heavy doses of this or that virus. Our makeup is what creates it. Jean Burkard was found dead with hanging from the back of a stool in A man has been arrested and charged with the murder of a 46-year-old retiree.Photographs of the body of Jean Burkard, who had been"spotted"by police officers on Wednesday, have been posted on the Internet. A recording of one of the officers'arrest reports can be heard following Burkard's death, October 14th. The policeman said Burkard"could do nothing but drop her."Burkard's body was found unresponsive in the 170 Humperdinck petrol station the next day.He was arrested and later placed in jail for outpatient treatment.A case of human sexual assault was also registered against the officer and his partner. The case can only be prosecuted under the Sexual Offences Act 2007, but DNA samples have been sent to the Alberta Attorney General for review.In a separate case, it's alleged paraplegic Jim McKay raped an unconscious woman in 2005 at McBride Park in Toronto, and police later said this was linked to drunk driving. by Claudio Confuschi Selling inlab | More information we can provide regarding the sale of the server is controlled here: www.openhabitets.com/switzerland/householdsale/otherconstraints.htm Selling customers Some customers will exchange their internal/client/monetary damages under the maintenance of intellectual property' Dissent rages between our exchange, customers and the online and retail supply chain. Open Hospitallers (ODF) XMPD (http://os.openhabitets.com/xpmdf.html) (QSS: www.social-network-service.com the forum for open-minded.org users, employees,reporters and volunteers) is prohibited from entering the services, and customers should contact DOIB at mail.openhabitets.com. To implorally demand the return of the service, an OUR response is required: LET'S FREQUENTLY LISTEN TO YOU IN SCALE OF CONTENT. However, the success of selling such necessary goods as the service, the materials provided (e.g., product making, design); the recovery (e.g., the repair of our premises, materials to be disposed of: removal of waste, etc); the exclusive sharing of all rights (including the right to reproduce) with the other users of the consumption'and not dumb down, adhering to the requirements of OLIF (Individual Benefits Page 1 80) (use only IF YOU WANT TO) which 05-16-2015 '31:12 4,267 (CADESÂ°.131-487, 23-7-2015, 96-67-16), 67-76 Publication(s): PERSONAL AGENDA ' OPEN OBSESSION AND
\end{samplebox}
\end{modelbox}
\newpage
\begin{modelbox}{Duo}
% 1M-sc
Sample Generative Perplexity: 68.9153; Sample Entropy: 5.4827
\begin{samplebox}
Power Premier League's outfit and subsequently won the 2007 national team spot. But Gerrard was sort of annoying unseasonably and was the best NYCFC could reply to for a few decades, but then a strike and a similar mass casualty flood in the early 2000s. The team got blown from scratch after six seasons. This could affect the future of NYCFC."The departure of FC CEO John Roberson, former vice-president of MLS soccer operations, has just dug a deep hole in the club, according to La FC. But MLS's senior side would have to have to insuff team that they were going to break with the strategy. To prevent a severe downturn in NYCFC on Saturday, Salvador the head chair of MLS's sports office, said this weekend is the time to make a'suit'on the club.'The MLS now pays equal time to all the business operations over MLS while the club begins to convince the fans that its business needs to succeed. This sort of treatment is clear to sink into tens of millions of dollars of losses in revenues, thereby damaging the financial standing of the club."More than even a slight fraction of the staff would be in a layoff by the last several months, Roberson said. This should allow for NYCFC riding the biggest attendance Wave in MLS history this season, undergirding discussions of several sell-outs of the David Beckham saga asking to pay up \$15 million or \$16 million until next season if the league wants to monetize the franchise's iconic figure. The sheer size of the team in the past decade is that underdog story to Major League Soccer fans. Surely one where the most important of all odds, the year 2016 contract fight over Frank Lampard inevitably played a big role in NYCFC getting to the MLS Cup. Probably the biggest factor to NYFC form was FC sent 7,335, wanting fans out ofage time, despite the fans by the very nature of itself saying no more for it. That dream, after the game resulted in their last hat trick, should still hold true in the history of NYCFC. NYFC will then venture in on the biggest game of their inaugural MLS season against league champion Real Salt Lake at home on March 5. After 16 matches between them, although NYCFC has a move a step back in MLS playoffs warmup, that final leg will have a bigger stage. For Sunday night's playoff round, NYCFC returns to Virgin/Fox Sports Network at \$5 ET. The rest of that stream takes place here If you haven't already, grab the next round of season tickets at WICH-N. Follow @CNYFC on Twitter and Instagram. It's you guys! Your fiancÃ©. Cowies were. Let's be in love. Obviously. This part is coming on this 1st July, New Year, so I'm going to write this one now. I can't wait since you will probably want to know what you saw myself for now. Honestly, though, I have definitely been writing about this for a couple weeks now. AMAFT BALRAVIS. This cool guy is a gentleman not much known who has started working out to millions more people, who is gonna be pursuing a client he hopes to move into (over the 'man' title perhaps? Think about this unreal romance: So you get down, dev fields and all the way under at a sketchy meeting with the agent, and then you call the agent up and admits that he is still working on a career no one will want to try. For over a year later, and he then shows up with you. (The 'mh mane' that I've officially seen is A Brokow), and then relationship develops, then eventually moves into a relationship with Sarah, the dog. You think you're writing about this now and they're actually at some point in your life. A PPO-PPO boy says strange things and Sarah is curious. It took about 3 Devs, but it got that out to light. Here are some links to the forums that talk about this. Let me say Ace. That word IS an incredibly powerful set of things, the kind of I held for granted when I wasn't willing to use things until they were meant for me, as well as just an up and stairs c ' Wibby you gutter from the old very word 'Acicoolyou' English 'the..:' That is where this guy is, and I don't know whether I do actually. Hey, I don't see where are..out there .?! Yeah, yeah, it's in. What did you think Acee.com was like when you saw it in the game world?
\end{samplebox}
\end{modelbox}
\newpage
\begin{modelbox}{MDLM}
% 1M-sc
Sample Generative Perplexity: 64.5802; Sample Entropy: 5.2111
\begin{samplebox}
counter for other situations. Setup [edit] Mitchell is seen off from the centre and in front of ahead of the ball carrier Conclusion [edit] In order to align his position with the ball carrier to accommodate his duties, teams insisted that Mitchell be able to, facilitate tackling and connect to the ball.[3] Vehicles [edit] Mitchell being used as defender Another position in a presence'denial formation is primarily worked on the edge of the defence. The entry carrier to cover the weight of the penetrating game that was being played from underneath him'this phase does not affect the change of the passing game before Mitch in the tackle''is set in advance of the brick carrier-position position at the outset, where Mitchell disrupts the safety at tighthead immediately after the tackle. The weekend manner in world formation has long taken a shift by having an entry player to the hooker displayed in the middle position with Mitchell inside his release. To complete a tackle, Mitchell was able to work the berrier and attack early when the hookor attacked from the distance into the five-yard line; the same or better, when the hooker had his side and the outside of the play and had his free-coming of the ball to the pivot or outside. This provides protection of the ability from a single penetrating tackle or around the pivot to create a wider gap. Mitchell's absence gives teams an extra layer of protection at this position, as the receiver in the territory could provide himself opening up the pass space to be higher than his teammate with the ball in hand. Another position where Mitchell's hands was not used set up the choice of position required in world formation, rather than simply defending over the other backrow carriers crossing the threshold from where he was initially placed.[3] In this defensive feature, other teams also featured the ball carrier as well as other assignments that indirectly affect alignment as several players player each on the formation.[1][3] See [edit] Post style rugby [edit] Analysis [edit] A panket comprises a group of other ball-follow pankets that within one defensive ministry. On these teams Mitchell would often use the edge of his foot as a pivot and then rotate until the player gets up to touch back.[3][4] The right side-volley provides the ground to the ball carrier, though the head is a running forward best used to his hips. The previous man is able to receive his ball from the outside of possession so the opponent receives the ball, sees out the offensive, debriefed on the positioning of the end player, chase points and the handler upfield delivering a back-to-back free-fall onwards despite the presence of the ball carrier. Common approaches do have Mitchell not as advanced up the outside, that especially has m-ley tackles usually with a dummy outside or center in a centre or base formation. By usually using this personnel they can try to establish a disparity in the opening line-up. Rather than bringing an LB who would carry the ball to the right side, two tackles and a nose would carry the dummy who would have controlled the pass from the inside.[1][5] This two maulleys approach always leaves the offensive-line open on the inside of the ball, and Mitchell sealing a tackle on the ball or completely carrying it back away generally leads into critical play.[6] This approach is much preferred beyond coverage, the parallel defence triump line where one side carries what has risen through, another covering the use of foot in which the match intends to intercept but where the line becomes left-footed rather than the line of defence. If the opposition would like to be more quickly held onto the football on the inside of the ball, it could also be tackled at time. Often touched by the maulley used to manry the shot, rather than by the forehead contact and therefore allowing the third player to touch the ball can create less turnovers, as this particular player could take the ball on the inside, but not without another player carries. Conclusion [edit] Beginning with possession meant that if he had delivered a direct target interception both sides had chance to exploit from the outside, failing to make the opposition uncomfortable with the interception. The two playing partners could have put the opposition out of position and above the outside of the Tuck but the wide-offer would then have been wondering what might happen inside of the tuck without him working the play.[7] Another main feature of Mitchell coming onto the field at the highest level was his ability to deliver consistently. The correct term for offensive runs off the ball is wide crossed (from the outside of possession), but rather crossed across or near the line, Mitchell could have used him knock the ball loose straight on the opponent, while running and attacking this ball a few to 10 yards behind the opposition, in plain sight.[8] He may have risen
\end{samplebox}
\end{modelbox}

%% file: checklist.tex
\section*{NeurIPS Paper Checklist}

\begin{enumerate}

\item {\bf Claims}
    \item[] Question: Do the main claims made in the abstract and introduction accurately reflect the paper's contributions and scope?
    \item[] Answer: \answerYes{} % Replace by \answerYes{}, \answerNo{}, or \answerNA{}.
    \item[] Justification: The abstract and introduction accurately summarize the paper's main methodological contributions, theoretical formulation, and empirical scope, with claims supported by the presented analyses and experiments.

\item {\bf Limitations}
    \item[] Question: Does the paper discuss the limitations of the work performed by the authors?
    \item[] Answer: \answerYes{} % Replace by \answerYes{}, \answerNo{}, or \answerNA{}.
    \item[] Justification: The paper discusses limitations in the Section~\ref{sec:conclusion}.

\item {\bf Theory assumptions and proofs}
    \item[] Question: For each theoretical result, does the paper provide the full set of assumptions and a complete (and correct) proof?
    \item[] Answer: \answerYes{} % Replace by \answerYes{}, \answerNo{}, or \answerNA{}.
    \item[] Justification: See \cref{app:proof_ode_ppl}.

    \item {\bf Experimental result reproducibility}
    \item[] Question: Does the paper fully disclose all the information needed to reproduce the main experimental results of the paper to the extent that it affects the main claims and/or conclusions of the paper (regardless of whether the code and data are provided or not)?
    \item[] Answer: \answerYes{} % Replace by \answerYes{}, \answerNo{}, or \answerNA{}.
    \item[] Justification: The paper provides the key implementation details, training setup, evaluation protocol, and baseline information needed to reproduce the main experimental results supporting its claims.

\item {\bf Open access to data and code}
    \item[] Question: Does the paper provide open access to the data and code, with sufficient instructions to faithfully reproduce the main experimental results, as described in supplemental material?
    \item[] Answer: \answerYes{} % Replace by \answerYes{}, \answerNo{}, or \answerNA{}.
    \item[] Justification: The code is available at \url{https://github.com/nealchen2003/LangFlow}. The datasets are already public.

\item {\bf Experimental setting/details}
    \item[] Question: Does the paper specify all the training and test details (e.g., data splits, hyperparameters, how they were chosen, type of optimizer) necessary to understand the results?
    \item[] Answer: \answerYes{} % Replace by \answerYes{}, \answerNo{}, or \answerNA{}.
    \item[] Justification: See \cref{app:design-choices}.

\item {\bf Experiment statistical significance}
    \item[] Question: Does the paper report error bars suitably and correctly defined or other appropriate information about the statistical significance of the experiments?
    \item[] Answer: \answerYes{} % Replace by \answerYes{}, \answerNo{}, or \answerNA{}.
    \item[] Justification: The paper reports two-standard-deviation error bars for Gen.\,PPL over multiple evaluation runs in Appendix~\ref{app:exp:eval}.

\item {\bf Experiments compute resources}
    \item[] Question: For each experiment, does the paper provide sufficient information on the computer resources (type of compute workers, memory, time of execution) needed to reproduce the experiments?
    \item[] Answer: \answerYes{} % Replace by \answerYes{}, \answerNo{}, or \answerNA{}.
    \item[] Justification: We have included the information in Appendix~\ref{app:exp:training}.
    
\item {\bf Code of ethics}
    \item[] Question: Does the research conducted in the paper conform, in every respect, with the NeurIPS Code of Ethics \url{https://neurips.cc/public/EthicsGuidelines}?
    \item[] Answer: \answerYes{} % Replace by \answerYes{}, \answerNo{}, or \answerNA{}.
    \item[] Justification: The work uses standard public language modeling datasets and does not involve human-subject experiments, crowdsourcing, or private data collection.

\item {\bf Broader impacts}
    \item[] Question: Does the paper discuss both potential positive societal impacts and negative societal impacts of the work performed?
    \item[] Answer: \answerNA{} % Replace by \answerYes{}, \answerNo{}, or \answerNA{}.
    \item[] Justification: The work is foundational algorithmic research, and does not introduce a deployed system, new dataset, or capability that creates a direct societal impact beyond those of existing language models.
    
\item {\bf Safeguards}
    \item[] Question: Does the paper describe safeguards that have been put in place for responsible release of data or models that have a high risk for misuse (e.g., pre-trained language models, image generators, or scraped datasets)?
    \item[] Answer: \answerNA{} % Replace by \answerYes{}, \answerNo{}, or \answerNA{}.
    \item[] Justification: The paper studies 130M-parameter language models on standard benchmarks and does not release a high-risk deployed model or new scraped dataset.

\item {\bf Licenses for existing assets}
    \item[] Question: Are the creators or original owners of assets (e.g., code, data, models), used in the paper, properly credited and are the license and terms of use explicitly mentioned and properly respected?
    \item[] Answer: \answerYes{} % Replace by \answerYes{}, \answerNo{}, or \answerNA{}.
    \item[] Justification: All assets are properly cited and their licenses are respected.

\item {\bf New assets}
    \item[] Question: Are new assets introduced in the paper well documented and is the documentation provided alongside the assets?
    \item[] Answer: \answerNA{} % Replace by \answerYes{}, \answerNo{}, or \answerNA{}.
    \item[] Justification: The paper does not release new assets.

\item {\bf Crowdsourcing and research with human subjects}
    \item[] Question: For crowdsourcing experiments and research with human subjects, does the paper include the full text of instructions given to participants and screenshots, if applicable, as well as details about compensation (if any)? 
    \item[] Answer: \answerNA{} % Replace by \answerYes{}, \answerNo{}, or \answerNA{}.
    \item[] Justification: The paper does not involve crowdsourcing or research with human subjects.

\item {\bf Institutional review board (IRB) approvals or equivalent for research with human subjects}
    \item[] Question: Does the paper describe potential risks incurred by study participants, whether such risks were disclosed to the subjects, and whether Institutional Review Board (IRB) approvals (or an equivalent approval/review based on the requirements of your country or institution) were obtained?
    \item[] Answer: \answerNA{} % Replace by \answerYes{}, \answerNo{}, or \answerNA{}.
    \item[] Justification: The paper does not involve crowdsourcing or research with human subjects.

\item {\bf Declaration of LLM usage}
    \item[] Question: Does the paper describe the usage of LLMs if it is an important, original, or non-standard component of the core methods in this research? Note that if the LLM is used only for writing, editing, or formatting purposes and does \emph{not} impact the core methodology, scientific rigor, or originality of the research, declaration is not required.
    %this research? 
    \item[] Answer: \answerNA{} % Replace by \answerYes{}, \answerNo{}, or \answerNA{}.
    \item[] Justification: The core method development in this research does not involve LLMs as any important, original, or non-standard components.

\end{enumerate}